\documentclass[11pt]{article}
\usepackage{xcolor}
\usepackage[margin=1in,top=1in,bottom=1in]{geometry}
\definecolor{lightblue}{HTML}{2970CC}
\definecolor{lightpurple}{HTML}{673147}
\definecolor{ForestGreen}{HTML}{FF5733}
\usepackage[colorlinks=true,linkcolor=lightblue,urlcolor=blue,citecolor=lightblue,breaklinks]{hyperref}
\usepackage[skip=0.5\baselineskip]{parskip}
\usepackage[font=small,labelfont=bf]{caption}
\usepackage[numbers, sort&compress]{natbib}
\usepackage{color}
\usepackage{enumerate}
\usepackage{enumitem}
\usepackage{graphicx}
\usepackage{varwidth}
\usepackage{mathrsfs}
\usepackage{mathtools}
\usepackage{subcaption}
\usepackage{overpic}
\usepackage{multirow}
\usepackage{multicol}
\usepackage{amsmath,amsthm,amssymb,colonequals,etoolbox}
\usepackage{amsfonts}       
\usepackage{thmtools}
\usepackage{url}
\usepackage[capitalise,noabbrev]{cleveref}
\crefformat{equation}{(#2#1#3)}
\crefrangeformat{equation}{(#3#1#4) to (#5#2#6)}
\crefmultiformat{equation}{(#2#1#3)}{ and (#2#1#3)}{, (#2#1#3)}{ and (#2#1#3)}
\usepackage[T1]{fontenc}    
\usepackage{nicefrac}       
\usepackage{microtype}      
\usepackage{xcolor}         
\usepackage{booktabs}
\usepackage{multirow}
\usepackage{makecell}
\usepackage{authblk}
\usepackage[ruled,vlined]{algorithm2e}
\crefname{algocf}{Algorithm}{algorithms}
\Crefname{algocf}{Algorithm}{Algorithms}
\usepackage{wrapfig}

\renewcommand{\geq}{\geqslant}
\renewcommand{\le}{\leqslant}
\renewcommand{\leq}{\leqslant}

\declaretheoremstyle[
  bodyfont=\itshape,
  headfont=\bfseries,
  spaceabove=6pt,
  spacebelow=6pt
]{spaced}
\declaretheorem[numberwithin=section,style=spaced]{theorem}
\declaretheorem[numberlike=theorem,style=spaced]{definition}

\declaretheorem[numberlike=theorem,style=spaced]{assumption}

\DeclareMathOperator*{\argmin}{arg\!\min}

\DeclareMathOperator{\Law}{Law}

\newcommand{\A}{\ensuremath{\mathcal{A}}}

\newcommand{\calL}{\mathcal{L}}

\newcommand{\R}{\ensuremath{\mathbb{R}}}

\newcommand{\N}{\ensuremath{\mathbb{N}}}

\newcommand{\E}{\mathbb{E}}

\newcommand{\ceil}[1]{\lceil #1 \rceil}

\newcommand{\Id}{\text{\it Id}}

\numberwithin{equation}{section}

\newcommand{\tmin}{t_{\min}}
\newcommand{\tmax}{t_{\max}}
\newcommand{\stopgrad}[1]{\mathsf{stopgrad}(#1)}

\newcommand{\lmd}{\mathsf{LMD}}
\newcommand{\emd}{\mathsf{EMD}}
\newcommand{\pfmm}{\mathsf{PFMM}}
\newcommand{\fmmlab}{\mathsf{FMM}}
\newcommand{\ee}{\mathsf{EE}}
\newcommand{\si}{\mathsf{SI}}

\title{\centering Flow map matching with stochastic interpolants:\\A mathematical framework for consistency models}
\author{Nicholas M.~Boffi$^*$}
\author{Michael S.~Albergo$^*$}
\author{Eric Vanden-Eijnden}
\affil{Courant Institute of Mathematical Sciences}

\begin{document}

\maketitle

\begin{abstract}
	Generative models based on dynamical equations such as flows and diffusions offer exceptional sample quality, but require computationally expensive numerical integration during inference.
The advent of consistency models has enabled efficient one-step or few-step generation, yet despite their practical success, a systematic understanding of their design has been hindered by the lack of a comprehensive theoretical framework.
Here we introduce Flow Map Matching (FMM), a principled framework for learning the two-time flow map of an underlying dynamical generative model, thereby providing this missing mathematical foundation.
Leveraging stochastic interpolants, we propose training objectives both for distillation from a pre-trained velocity field and for direct training of a flow map over an interpolant or a forward diffusion process.
Theoretically, we show that FMM unifies and extends a broad class of existing approaches for fast sampling, including consistency models, consistency trajectory models, and progressive distillation.
Experiments on CIFAR-10 and ImageNet-32 highlight that our approach can achieve sample quality comparable to flow matching while reducing generation time by a factor of 10-20.

\end{abstract}

\section{Introduction}
\label{sec:intro}
In recent years, diffusion models~\citep{song2020score, ho2020denoising, sohldickstein2015deep, song_improved_2020, song_generative_2020} have achieved state of the art performance across diverse modalities, including image~\citep{dhariwal_diffusion_2021, rombach2022high, esser_scaling_2024}, audio~\citep{popov_grad-tts_2021, jeong_diff-tts_2021, huang_prodiff_2022, lu_conditional_2022}, and video~\citep{ho_imagen_2022, ho_video_2022, blattmann_align_2023, wu_tune--video_2023}.
These models belong to a broader class of approaches including flow matching~\citep{lipman2022flow}, rectified flow~\citep{liu2022flow}, and stochastic interpolants~\citep{albergo2022building, albergo2023stochastic}, which construct a path in the space of measures between a base and a target distribution by specifying an explicit mapping between samples from each.
This construction yields a dynamical transport equation governing the evolution of the time-dependent probability measure along the path.
The generative modeling problem then reduces to learning a velocity field that accomplishes this transport, which provides an efficient and stable training paradigm.

At sample generation time, models in this class generate data by smoothly transforming samples from the base into samples from the target through numerical integration of a differential equation.
While effective, the number of integration steps required to produce high-quality samples incurs a cost that can limit real-time applications~\citep{chi_diffusion_2024}.
Comparatively, one-step models such as generative adversarial networks~\citep{goodfellow_generative_2014, goodfellow_generative_2020, creswell_generative_2018} are notoriously difficult to train~\citep{metz_unrolled_2017, arjovsky_wasserstein_2017}, but can be orders of magnitude more efficient to sample, because they only require a single network evaluation.
As a result, there has been significant recent research effort focused on maintaining the stable training of diffusions while reducing the computational burden of inference~\citep{karras_elucidating_2022}.
\begin{figure}[t]
	\centering
	\includegraphics[width=0.85\textwidth]{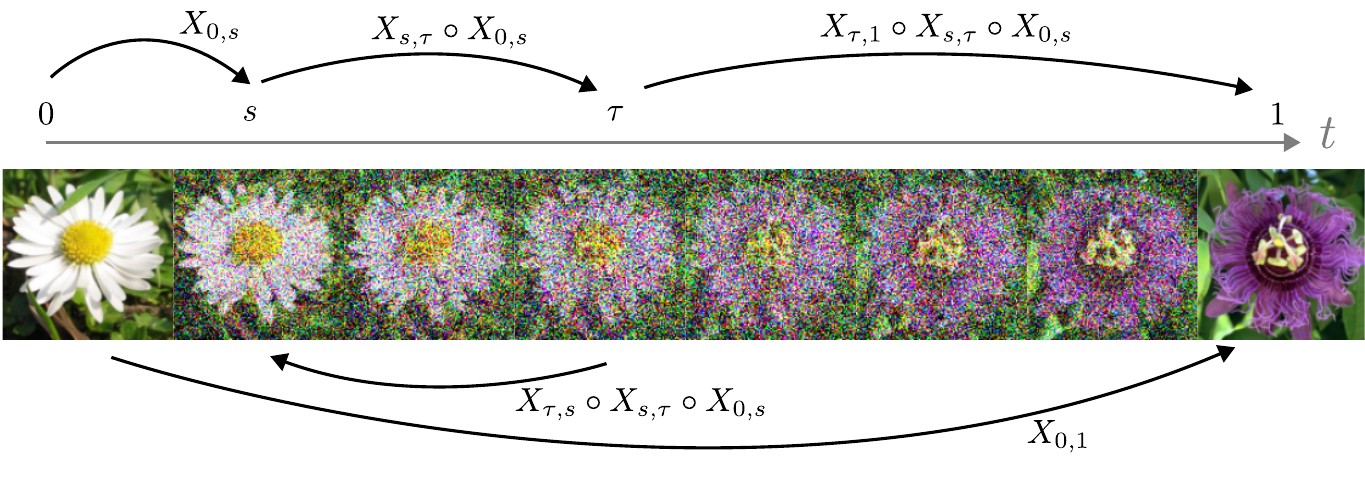}
	\caption{
		\textbf{Overview of Flow Map Matching.}
		Our approach learns the two-time flow map $X_{s, t}$ that transports the solution of an ordinary differential equation from time $s$ to time $t$.
		Unlike methods that learn instantaneous velocity fields, this bidirectional map can be used to build an integrator with arbitrary discretization.
		The integrator is exact in theory, and --crucially-- its number of steps can be adjusted post-training to balance accuracy and computational efficiency.
		The flow map can be distilled from a known velocity field or learned directly, and supports arbitrary base distributions, as illustrated here with image-to-image translation.
	}
	\label{fig:overview}
\end{figure}
Towards this goal, we introduce Flow Map Matching---a theoretical framework for learning the two-time flow map of a probability flow.
Specifically, our primary contributions can be summarized as:

\begin{itemize}[leftmargin=0.2in]
	\item \textbf{A mathematical framework.}
	      We characterize the flow map theoretically by identifying its fundamental mathematical properties, which we show immediately lead to distillation-based algorithms for estimating the flow map from a pre-trained velocity field.
	      Using our characterization, we establish connections between our framework and recent approaches for learning consistency models~\citep{song_consistency_2023, song_improved_2023, kim_consistency_2024} and for progressive distillation of a diffusion model~\citep{salimans_progressive_2022, zheng_fast_2023}.
	\item \textbf{Lagrangian vs Eulerian learning.}
	      %
	      We develop a novel \textit{Lagrangian} objective for distilling a flow map from a pre-trained velocity field, which performs well across our experiments. This loss is complementary to a related \textit{Eulerian} loss, which we prove is the continuous-time limit of consistency distillation~\citep{song_consistency_2023}. We show that both the Lagrangian and Eulerian losses control the Wasserstein distance between the teacher and the student models.
	\item \textbf{Direct training.}
	      We extend our Lagrangian loss to design a novel direct training objective for flow maps that eliminates the need for a pre-trained velocity field.
	      We also discuss how to modify the Eulerian loss to allow for direct training of the map, and identify some caveats with this procedure.
	      Our approach elucidates how the standard training signal for flow matching can be adapted to capture the more global signal needed to train the flow map.
	\item \textbf{Numerical experiments.}
	      We validate our theoretical framework through experiments on CIFAR-10 and ImageNet-$32$, where we achieve high sample quality at significantly reduced computational cost.
	      In particular, we highlight that learning the two-time flow map enables post-training adjustment of sampling steps, allowing practitioners to systematically trade accuracy for computational efficiency.
\end{itemize}

\section{Related Work}
\label{sec:related}
\paragraph{Dynamical transport of measure.}
Our approach is built upon the modern perspective of generative modeling based on dynamical transport of measure.
Grounded in the theory of optimal transport~\citep{villani2009optimal, benamou2000computational, santambrogio2015optimal}, these models originate at least with~\citep{tabak_density_2010, tabak2013}, but have been further developed by the machine learning community in recent years~\citep{rezende_variational_2015, dinh_density_2017, grathwohl_ffjord, chen_neural_2019}.
A breakthrough in this area originated with the appearance of score-based diffusion models~\citep{song2020score, song_improved_2020, song_generative_2020}, along with related denoising diffusion probabilistic models~\citep{ho2020denoising, sohldickstein2015deep}.
These methods generate samples by learning to time-reverse a stochastic differential equation with stationary density given by a Gaussian.
More recent approaches such as flow matching~\citep{lipman2022flow}, rectified flow~\citep{liu2022flow, liu2022let}, and stochastic interpolants~\citep{albergo2022building, albergo2023stochastic, ma_sit_2024, chen_probabilistic_2024} similarly construct connections between the base density and the target, but allow for arbitrary base densities and provide a greater degree of flexibility in the construction of the connection.

\paragraph{Reducing simulation costs.}
There has been significant recent interest in reducing the cost associated with solving an ODE or an SDE for a generative model based on dynamical measure transport.
One such approach, pioneered by rectification~\citep{liu2022flow, liu2022let}, is to try to straighten the paths of the probability flow, so as to enable more efficient adaptive integration.
In the limit of optimal transport, the paths become straight lines and the integration can be performed in a single step.
A second approach is to introduce \textit{couplings} between the base and the target, such as by computing the optimal transport over a minibatch~\citep{pooladian2023multisample, tong_improving_2023}, or by using data-dependent couplings~\citep{albergo_stochastic_dd_2023}, which can simplify both training and sampling.
A third approach has been to design hand-crafted numerical solvers tailored for diffusion models~\citep{karras_elucidating_2022, zhang_fast_2023, jolicoeur-martineau_gotta_2021, liu_pseudo_2022, lu_dpm-solver_nodate}, or to learn these solvers directly~\citep{watson_learning_2021, watson_learning_2022, nichol_improved_2021} to maximize efficiency.
\textit{Instead, we propose to learn the flow map directly, which avoids estimating optimal transport maps and can overcome the inherent limitations of numerical integration.}

\paragraph{Consistency models.}
Most related to our approach are the classes of one-step models based on \textit{distillation} or \textit{consistency}; we give an explicit mapping between these techniques and our own in~\cref{sec:app:relation}.
Consistency models~\citep{song_consistency_2023} have been introduced as a new class of generative models that can either be distilled from a pre-trained diffusion model or trained directly, and are related to several notions of \textit{consistency} of the score model that have appeared in the literature~\citep{lai_equivalence_2023, lai_improving_2023, shen_self-consistency_2022, boffi_probability_2023, daras_consistent_2023}.
These models learn a one-step map from noise to data, and can be seen as learning a single-time flow map.
While they can perform very well, consistency models do not benefit from multistep sampling, and exhibit training difficulties that mandate delicate hyperparameter tuning~\citep{song_improved_2023}.
By contrast, we learn a two-time flow map, which enables us to smoothly benefit from multistep sampling.

\paragraph{Extensions of consistency models.}
Consistency trajectory models~\citep{kim_consistency_2024} were later introduced to improve multistep sampling and to enable the student to surpass the performance of the teacher.
Similar to our approach, these models learn a two-time flow map, but do so using a very different loss that incorporates challenging adversarial training.
Generalized consistency trajectory models~\citep{kim_generalized_2024} extend this approach to the stochastic interpolant setting, but use the consistency trajectory loss, and do not introduce the Lagrangian perspective considered here.
Bidirectional consistency models~\citep{li_bidirectional_2024} learn a two-time invertible flow map similar to our method, but do so in the score-based diffusion setting, and do not leverage our Lagrangian approach; all of these existing consistency models can be seen as a special case of our formalism.

\paragraph{Progressive and operator distillation.}
Neural operator approaches~\citep{zheng_fast_2023} learn a one-time flow map from noise to data, but do so by first generating a dataset of trajectories from the probability flow.
Progressive distillation~\citep{salimans_progressive_2022} and knowledge distillation~\citep{luhman_knowledge_2021} techniques aim to convert a diffusion model into an equivalent model with fewer samples by matching several steps of the original diffusion model.
These approaches are related to our flow map distillation scheme, though the object we distill is fundamentally different.

\section{Flow Map Matching}
\label{sec:theory}

The central object in our method is the \textit{flow map}, which maps points along trajectories of solutions to an ordinary differential equation (ODE).
Our focus in this work is primarily on probability flow ODEs that arise in generative models, such as those constructed using flow matching via stochastic interpolants or score-based diffusion models (see~\cref{app:si_diff} for a concise, self-contained review of these approaches).
While many of our results apply more generally to other ODEs, we present our definitions and theoretical results in this specific context to highlight their relevance to generative modeling.
All proofs of the statements made in this section are provided in~\cref{app:theory}, with some additional theoretical results and connections to existing consistency and distillation techniques given in~\Cref{sec:app:relation}.

\subsection{Stochastic interpolants and probability flows}
\label{sec:si:prob:flow}

We first recall the main components of the stochastic interpolant framework~\citep{albergo2022building,albergo2023stochastic}, which we use to construct the probability flow ODEs central to our study.
We begin by giving the definition of a stochastic interpolant.
\begin{restatable}[Stochastic Interpolant]{definition}{sidef}
	\label{def:stoch_interp}
	The stochastic interpolant $I_t$ between probability densities~$\rho_0$ and~$\rho_1$ is the stochastic process given by
	\begin{align}
		\label{eq:stoch:interp}
		I_t = \alpha_t x_0 + \beta_t x_1 + \gamma_t z,
	\end{align}
	where $\alpha, \beta, \gamma^2 \in C^1([0, 1])$ satisfy $\alpha_0 = \beta_1 = 1$, $\alpha_1 = \beta_0 = 0$, and $\gamma_0 = \gamma_1 = 0$.
	In~\eqref{eq:stoch:interp}, $(x_0, x_1)$ is drawn from a coupling $(x_0, x_1) \sim \rho(x_0, x_1)$ that satisfies the marginal constraints $\int_{\R^d}\rho(x_0, x_1) dx_0 = \rho_1(x_1)$ and $\int_{\R^d}\rho(x_0, x_1) dx_1 = \rho_0(x_0)$.
	Moreover, $z \sim \mathsf{N}(0, \text{Id})$ with $z\perp(x_0,x_1)$.
\end{restatable}
\noindent
Theorem 3.6 of~\cite{albergo2023stochastic} shows that the stochastic interpolant given in~\cref{def:stoch_interp} specifies an underlying probability flow ODE, as we now recall.
\begin{restatable}[Probability Flow]{proposition}{pflow}
	\label{thm:stoch_interp}
	For all $t\in[0,1]$, the probability density of $I_t$ is the same as the probability density of the solution to
	\begin{equation}
		\label{eqn:ode}
		\dot{x}_t = b_t(x_t), \qquad x_{t=0} = x_0 \sim \rho_0,
	\end{equation}
	where $b: [0,1]\times \R^d\rightarrow\R^d$ is the time-dependent velocity field (or drift)  given by
	\begin{equation}
		\label{eqn:b:def}
		b_t(x) = \E[\dot{I}_t | I_t = x].
	\end{equation}
	In~\eqref{eqn:b:def}, $\E[\:\:\cdot\:\:|I_t = x]$ denotes an expectation over the coupling $(x_0, x_1) \sim \rho(x_0, x_1)$ and $z \sim \mathsf N(0, \Id)$ conditioned on the event $I_t=x$.
\end{restatable}
Although \eqref{eqn:ode} is deterministic for any single trajectory, taken together, its solutions are random because the initial conditions are sampled from $\rho_0$.
We denote by $\rho_t=\text{Law}(x_t)$ the density of these solutions at time $t$.
The drift  $b$ can be learned efficiently in practice by solving a square loss regression problem~\citep{albergo2023stochastic}
\begin{equation}
	\label{eq:loss:b}
	b = \argmin_{\hat b} \int_0^1 \E \big[ | \hat b_t(I_t) - \dot I_t|^2 \big] dt,
\end{equation}
where $\E$ denotes an expectation over the coupling $(x_0, x_1) \sim \rho(x_0, x_1)$ and $z \sim \mathsf N(0, \Id)$.

A canonical choice when $\rho_0 = \mathsf{N}(0, \Id)$ considered in~\citet{albergo2022building} corresponds to $\alpha_t = 1-t$, $\beta_t = t$, and $\gamma_t = 0$, which recovers flow matching~\citep{lipman2022flow} and rectified flow~\citep{liu2022flow}.
The choice $\alpha_t = 0$, $\beta_t = t$ and $\gamma_t = \sqrt{1-t^2}$ corresponds to a variance-preserving diffusion model with the time-rescaling $t = -\log\tau$ where $\tau \in [0, \infty)$ is the usual diffusion time\footnote{Note that $\gamma_0=1$ in this case, so that $I_0=z$}.
A variance-exploding diffusion model may be obtained by taking $\alpha_t = 0$, $\beta_t = 1$, and $\gamma_t = T-t$ with $t \in [0, 1]$ and where $\tau = T-t$ is the usual diffusion time, though this violates the boundary conditions in~\cref{def:stoch_interp}. For more details about the connection between stochastic interpolants and diffusion models, we again refer the reader to~\cref{{app:si_diff}}.

\subsection{Flow map: definition and characterizations}
\label{ssec:setup}
To generate a sample from the target density $\rho_1$, we can draw an initial point from $\rho_0$ and numerically integrate the probability flow ODE~\eqref{eqn:ode} over the interval $t\in[0,1]$.
While this approach produces high-quality samples, it typically requires numerous integration steps, making inference computationally expensive---particularly when $b_t$ is parameterized by a complex neural network.
Here, we bypass this numerical integration by estimating the two-time flow map, which lets us take jumps of arbitrary size.
To do so, we require the following regularity assumption, which ensures~\eqref{eqn:ode} has solutions that exist and are unique~\citep{hartman2002ordinary}.

\begin{assumption}
	\label{as:one-sided}
	The drift satisfies the one-sided Lipschitz condition
	\begin{equation}
		\label{eq:one:sided:lip}
		\exists \:\: C_t \in L^1[0,1] \ : \quad (b_t(x)-b_t(y))\cdot (x-y) \le C_t |x-y|^2 \quad \text{for all} \:\: (t,x,y) \in [0,1]\times \R^d \times \R^d.
	\end{equation}
\end{assumption}
With~\cref{as:one-sided} in hand, we may now define the central object of our study.
\begin{definition}[Flow Map]
	\label{def:flow_map}
	The flow map $X_{s, t}:\R^d\rightarrow\R^d$ for~\eqref{eqn:ode} is the unique map such that
	\begin{equation}
		\label{eqn:flow_map}
		X_{s, t}(x_s) = x_t \:\: \text{for all} \:\: (s,t) \in [0,1]^2,
	\end{equation}
	where $(x_t)_{t\in[0,1]}$ is any solution to the ODE~\eqref{eqn:ode}.
\end{definition}
The flow map in~\cref{def:flow_map} can be seen as an integrator for~\eqref{eqn:ode} where the step size $t-s$ may be chosen arbitrarily.
In addition to the defining condition~\eqref{eqn:flow_map}, we now highlight some of its useful properties.
\begin{restatable}{proposition}{flowmap}
	\label{prop:flow_map}
	The flow map $X_{s, t}(x)$ is the unique solution to the Lagrangian equation
	\begin{equation}
		\label{eqn:flow_map_dyn:L}
		\partial_t X_{s, t}(x) = b_t(X_{s, t}(x)), \qquad X_{s,s}(x) = x,
	\end{equation}
	for all $(s,t,x) \in [0,1]^2\times \R^d$. In addition, it satisfies
	\begin{equation}
		\label{eq:semigroup}
		X_{t,\tau}(X_{s,t}(x)) = X_{s,\tau}(x)
	\end{equation}
	for all $(s,t,\tau,x) \in [0,1]^3\times \R^d$.
	In particular $X_{s,t} (X_{t,s}(x)) = x$ for all $(s,t,x) \in [0,1]^2\times \R^d$, i.e. the flow map is invertible.
\end{restatable}
\cref{prop:flow_map} shows that, given an $x_0\sim\rho_0$, we can use the flow map to generate samples from $\rho_t$ for any $t\in[0,1]$ in one step via $x_t=X_{0,t}(x_0)\sim \rho_t$.
The composition relation~\eqref{eq:semigroup} is the \emph{consistency property}, here stated over two times, that gives consistency models their name~\citep{song_consistency_2023}.
This relation shows that we can also generate samples in multiple steps using $x_{t_k} = X_{t_{k-1},t_k}(x_{t_{k-1}})\sim\rho_{t_k}$ for any set of discretization points $(t_0,\ldots, t_K)$ with $t_k \in [0,1]$ and $K \in \N$.
We refer to~\eqref{eqn:flow_map_dyn:L} as the \emph{Lagrangian equation} because it is defined in a frame of reference that moves with  $X_{s, t}(x)$.

The flow map $X_{s, t}$ also obeys an alternative \emph{Eulerian equation} that is defined at any fixed point $x \in \R^d$ and which involves a derivative with respect to $s$. To derive this equation, note that since $X_{s,t}(x_s) = x_t$ by \eqref{eqn:flow_map} we have $(d/ds)X_{s,t}(x_s) =0$, which by the chain rule can be written as
\begin{equation}
	\label{eqn:backward:2}
	\frac{d}{ds} X_{s,t} (x_s) = \partial_s X_{s, t}(x_s) + b_s(x_s)\cdot \nabla X_{s, t}(x_s) = 0
\end{equation}
where we used \cref{eqn:ode} to set $\dot x_s = b_s(x_s)$.
Evaluating this equation at $x_s=x$ gives the announced result.
\begin{restatable}{proposition}{euleq}
	\label{prop:prop_backwards}
	The flow map $X_{s, t}$ is the unique solution of the Eulerian equation,
	\begin{equation}
		\label{eqn:backward}
		\partial_s X_{s, t}(x) + b_s(x)\cdot \nabla X_{s, t}(x) = 0, \qquad X_{t,t} (x) = x,
	\end{equation}
	for all $(s,t,x) \in [0,1]^2\times \R^d$.
\end{restatable}
\subsection{Distillation of a known velocity field}
\label{ssec:distill}
The Lagrangian equation~\eqref{eqn:flow_map_dyn:L} in~\Cref{prop:flow_map} leads to a distillation loss that can be used to learn a flow map for a probability flow with known right-hand side~$b$, as we now show.

\begin{restatable}[Lagrangian map distillation]{corollary}{lagdistill}
	\label{prop:distill}
	Let $w_{s,t} \in L^1([0,1]^2)$ be a weight function satisfying $w_{s, t} > 0$ and let $I_s$ be the stochastic interpolant defined in~\cref{eq:stoch:interp}.
	Then the flow map is the global minimizer over $\hat{X}$ of the loss
	\begin{equation}
		\label{eqn:distillation_loss}
		\calL_{\lmd}(\hat{X}) = \int_{[0,1]^2}  w_{s,t} \E \big[ |\partial_t \hat{X}_{s, t}(I_s) - b_t(\hat{X}_{s, t}(I_s)|^2\big] dsdt,
	\end{equation}
	subject to the boundary condition that $\hat{X}_{s, s}(x) = x$ for all $x \in \R^d$ and $s \in [0,1]$. In~\eqref{eqn:distillation_loss}, $\E$ denotes an expectation over the coupling $(x_0, x_1) \sim \rho(x_0, x_1)$ and $z \sim \mathsf N(0, \Id)$.
\end{restatable}
The time integrals in~\eqref{eqn:distillation_loss} can also be written as an expectation over $(s,t) \sim w_{s,t}$ (properly normalized):
\begin{equation}
	\label{eqn:distillation_loss:2}
	\calL_{\lmd}(\hat{X}) = \E_{(s,t)\sim w_{s,t}} \E\big[|\partial_t \hat{X}_{s, t}(I_s) - b_t(\hat{X}_{s, t}(I_s))|^2\big].
\end{equation}
We also note that the result of \Cref{prop:distill} remains true if we replace $I_s$ by any process $\hat x_s$ with density $\hat \rho_s \not= \rho_s$, as long as it is strictly positive everywhere.
In practice, it is convenient to use $I_s$ as it guarantees that we learn the flow map where we typically need to evaluate it.
Moreover, using $I_s$ avoids the need to generate a dataset from a pre-trained model, so that the objective~\eqref{eqn:distillation_loss} can be evaluated empirically in a simulation-free fashion.

When applied to a pre-trained flow model,~\Cref{prop:distill} can be used to train a new, few-step generative model with performance that matches the performance of the teacher.
When $\hat{X}_{s, t}$ is parameterized by a neural network, the partial derivative with respect to $t$ can be computed efficiently at the same time as the computation of $\hat{X}_{s, t
	}$ using forward-mode automatic differentiation.
This procedure is summarized in~\cref{alg:lagrange_distill}.

For simplicity, \cref{prop:distill} is stated for $w_{s,t}>0$ (e.g. $w_{s,t}=1$), so that we can estimate the map $X_{s,t}$ and its inverse $X_{t,s}$ for all $(s,t)\in[0,1]$.
Nevertheless, this weight can also be adjusted to learn the map for different pairs $(s,t)$ of interest.
For example, if we only want to estimate the forward map with $s\le t$, then we can set $w_{s,t} = 1$ if  $s\le t$ and $w_{s,t} = 0$ otherwise.

By squaring the left hand-side of the Eulerian equation~\eqref{eqn:backward} in \Cref{prop:prop_backwards}, we may construct a second loss function for distillation.\footnote{In~\eqref{eqn:backward}, the term
$\left(b_s(x)\cdot\nabla X_{s, t}(x)\right)_i = \sum_{j=1}^d [b_s(x)]_j \partial_{x_j}\left[X_{s, t}(x)\right]_i = \left[\nabla X_{s, t}(x)\cdot b_s(x)\right]_i$
corresponds to a Jacobian-vector product that can be computed efficiently using forward-mode automatic differentiation.}
\begin{restatable}[Eulerian map distillation]{corollary}{euldistill}
	\label{prop:backward_loss}
	Let $w_{s,t}$ and $I_s$ be as in \cref{prop:distill}. Then the flow map is the global minimizer over $\hat{X}$ of the loss
	\begin{equation}
		\label{eqn:L_backward}
		\calL_{\emd}(\hat{X}) = \int_{[0,1]^2}w_{s,t} \E\big[|\partial_s \hat{X}_{s, t}(I_s) + b_s(I_s) \cdot \nabla \hat{X}_{s, t}(I_s)|^2\big] dsdt,
	\end{equation}
	subject to the boundary condition $\hat{X}_{s, s}(x) = x$ for all $x \in \R^d$ and for all $s \in \R$.
\end{restatable}
As in the discussion after~\cref{prop:distill}, here we may also replace $I_s$ by any process $\hat x_s$ with a strictly positive density.
Writing the time integrals in~\eqref{eqn:L_backward} as an expectation over $(s,t) \sim w_{s,t}$ gives
\begin{equation}
	\label{eqn:L_backward:b}
	\calL_{\lmd}(\hat{X}) = \E_{(s,t)\sim w_{s,t}} \E\big[ \big|\partial_s \hat{X}_{s, t}(I_s) + b_s(I_s) \cdot \nabla \hat{X}_{s, t}(I_s)\big|^2\big].
\end{equation}
We summarize a training procedure based on~\cref{prop:backward_loss} in~\cref{alg:euler_distill}.

In~\Cref{sec:app:relation}, we demonstrate how the preceding results connect with existing distillation-based approaches.
In particular, when $b_t(x)$ is the velocity of the probability flow ODE associated with a diffusion model, \Cref{prop:backward_loss} recovers the continuous-time limit of consistency distillation~\citep{song_consistency_2023, song_improved_2023} and consistency trajectory models~\citep{kim_consistency_2024}, while~\cref{prop:distill} is new.

\begin{algorithm}[t]
	\caption{Lagrangian map distillation (LMD)}
	\label{alg:lagrange_distill}
	\SetAlgoLined
	\SetKwInput{Input}{input}
	\Input{Interpolant coefficients $\alpha_t, \beta_t, \gamma_t$; pre-trained velocity $b_t$, weight function $w_{s, t}$, batch size $M$.}
	\Repeat{converged}{
		Draw batch $(s_i, t_i, x_0^i, x_1^i, z_i)_{i=1}^M$ from $w_{s,t} \times \rho(x_0,x_1) \times \mathsf N(z; 0, \Id)$.\\
		Compute $I_{t_i} = \alpha_{t_i} x_0^i + \beta_{t_i} x_1^i + \gamma_{t_i} z_i$ and $\dot I_{t_i} = \dot\alpha_{t_i} x_0^i + \dot\beta_{t_i} x_1^i + \dot\gamma_{t_i} z_i$.\\
		Compute $\hat{\calL}_{\lmd} = \frac{1}{M}\sum_{i=1}^M |\partial_t \hat{X}_{s_i, t_i}(I_{s_i}) - b_{t_i}(\hat{X}_{s_i, t_i}(I_{s_i}))|^2$.\\
		Take gradient step on $\hat{\calL}_{\lmd}$ to update $\hat X$.
	}
	\SetKwInput{Output}{output}
	\Output{Flow map $\hat{X}$.}
\end{algorithm}

\begin{algorithm}[t]
	\caption{Eulerian map distillation (EMD)}
	\label{alg:euler_distill}
	\SetAlgoLined
	\SetKwInput{Input}{input}
	\Input{Interpolant coefficients $\alpha_t, \beta_t, \gamma_t$; pre-trained velocity $b_t$, weight function $w_{s, t}$, batch size $M$.}
	\Repeat{converged}{
		Draw batch $(s_i, t_i, x_0^i, x_1^i, z_i)_{i=1}^M$ from $w_{s,t} \times \rho(x_0,x_1) \times \mathsf N(z; 0, \Id)$.\\
		Compute $I_{t_i} = \alpha_{t_i} x_0^i + \beta_{t_i} x_1^i + \gamma_{t_i} z_i$ and $\dot I_{t_i} = \dot\alpha_{t_i} x_0^i + \dot\beta_{t_i} x_1^i + \dot\gamma_{t_i} z_i$.\\
		Compute $\hat{\calL}_{\emd} = \frac{1}{M}\sum_{i=1}^M |\partial_s \hat{X}_{s_i, t_i}(I_{s_i}) + \nabla\hat{X}_{s_i, t_i}(I_{s_i}) b_{s_i}(I_{s_i})|^2$.\\
		Take gradient step on $\hat{\calL}_{\emd}$ to update $\hat X$.
	}
	\SetKwInput{Output}{output}
	\Output{Flow map $\hat{X}$.}
\end{algorithm}

\subsection{Wasserstein control}
\label{ssec:wasserstein}
In this section, we show that the Lagrangian and Eulerian distillation losses~\eqref{eqn:distillation_loss} and~\eqref{eqn:L_backward} control the Wasserstein distance between the density $\rho_t$ of the teacher flow model and the density $\hat{\rho}_t = \hat{X}_{0,t}\:\sharp\:\rho_0$ of the pushforward of $\rho_0$ under the learned flow map (that is, $\hat \rho_t$ is the density of $\hat X_{0,t}(x_0)$ with $x_0 \sim \rho_0$).
When combined with the Wasserstein bound in~\citet{albergo2022building}, the following results also imply a bound on the Wasserstein distance between the data density and the pushforward density for the learned flow map in the case where $b$ is a pre-trained stochastic interpolant or diffusion model.
We begin by stating our result for Lagrangian distillation.
\begin{restatable}[Lagrangian error bound]{proposition}{lagrangianerr}
	\label{prop:lagrange_wasserstein}
	Let $X_{s, t}: \R^d\rightarrow \R^d$ denote the flow map for $b$, and let $\hat{X}_{s, t}: \R^d \rightarrow \R^d$ denote an approximate flow map.
	Given $x_0\sim\rho_0$, let $\hat{\rho}_1$ be the density of $\hat X_{0,1}(x_0)$ and let $\rho_1$ be the target density of $X_{0,1}(x_0)$.
	Then,
	\begin{equation}
		\label{eq:W2:1}
		W_2^2(\rho_1, \hat{\rho}_1) \leq e^{1+2\int_0^1|C_t|dt}\calL_{\lmd}(\hat{X}).
	\end{equation}
	where $C_t$ is the constant that appears in Assumption~\ref{as:one-sided}.
\end{restatable}
The proof is given in~\cref{app:theory}.
We now state an analogous result for the Eulerian case.
\begin{restatable}[Eulerian error bound]{proposition}{eulerianerr}
	\label{prop:euler_wasserstein}
	Let $X_{s, t}: \R^d\rightarrow \R^d$ denote the flow map for $b$, and let $\hat{X}_{s, t}: \R^d \rightarrow \R^d$ denote an approximate flow map.
	Given $x_0\sim\rho_0$, let $\hat{\rho}_1$ be the density of $\hat X_{0,1}(x_0)$ and let $\rho_1$ be the target density of $X_{0,1}(x_0)$.
	Then,
	\begin{equation}
		\label{eq:W2:2}
		W_2^2(\rho_1, \hat{\rho}_1) \leq e^1 \calL_{\emd}(\hat{X}).
	\end{equation}
\end{restatable}

The proof is also given in~\cref{app:theory}.
The result in~\cref{prop:euler_wasserstein} appears stronger than the result in~\cref{prop:lagrange_wasserstein}, because it is independent of any Lipschitz constant.
Notice, however, that unlike~\eqref{eq:W2:1} the bound~\eqref{eq:W2:2} involves the spatial gradient of the map $X_{s,t}$, which may be more difficult to control.
In our numerical experiments, we found the best performance when using the Lagrangian distillation loss, rather than the Eulerian distillation loss.
We hypothesize and provide numerical evidence that this originates from avoiding the spatial gradient present in the Eulerian distillation loss; in several cases of interest, the learned map can be singular or nearly singular, so that the spatial gradient is not well defined everywhere.
This leads to training difficulties that manifest themselves as fuzzy boundaries on the checkerboard dataset and blurry images on image datasets.

\subsection{Direct training with flow map matching (FMM)}
\label{ssec:direct}
We now give a loss function for \emph{direct} training of the flow map that does not require a pre-trained $b$.

\begin{restatable}[Flow map matching]{proposition}{fmm}
	\label{prop:fmm}
	The flow map is the global minimizer over $\hat{X}$ of the loss
	\begin{align}
		\label{eq:obj:match}
		\calL_{\mathsf{FMM}}(\hat X) = \int_{[0,1]^2} w_{s,t} \left(\E \big [| \partial_t \hat X_{s,t} (\hat X_{t,s}(I_t)) - \dot{I}_t|^2\big] + \E \big [| \hat X_{s,t} (\hat X_{t,s}(I_t)) - I_t|^2 \big]\right) ds dt,
	\end{align}
	subject to the boundary condition $\hat{X}_{s, s}(x) = x$ for all $x \in \R^d$ and for all $s \in \R$.
	In~\eqref{eq:obj:match}, $w_{s,t} > 0$ and $\E$ is taken over the coupling $(x_0, x_1) \sim \rho(x_0, x_1)$ and $z \sim \mathsf N(0, \Id)$.
\end{restatable}

\begin{algorithm}[t]
	\caption{Flow map matching (FMM)}
	\label{alg:fmm}
	\SetAlgoLined
	\SetKwInput{Input}{input}
	\Input{Interpolant coefficients $\alpha_t, \beta_t, \gamma_t$; weight function $w_{s, t}$; batch size $M$.}
	\Repeat{converged}{
	Draw batch $(s_i, t_i, x_0^i, x_1^i, z_i)_{i=1}^M$ from $w_{s,t} \times \rho(x_0,x_1) \times \mathsf N(z; 0, \Id)$.\\
	Compute $I_{t_i} = \alpha_{t_i} x_0^i + \beta_{t_i} x_1^i + \gamma_{t_i} z_i$ and $\dot I_{t_i} = \dot\alpha_{t_i} x_0^i + \dot\beta_{t_i} x_1^i + \dot\gamma_{t_i} z_i$.\\
	Compute $\hat{\calL}_{\text{FMM}} = \frac{1}{M}\sum_{i=1}^M \left(|\partial_t \hat{X}_{s_i, t_i}(\hat{X}_{t_i, s_i}(I_{t_i})) - \dot{I}_{t_i}|^2 + |\hat{X}_{s_i, t_i}(\hat{X}_{t_i, s_i}(I_{t_i})) - I_{t_i}|^2\right)$.\\
	Take gradient step on $\hat{\calL}_{\text{FMM}}$ to update $\hat X$.
	}
	\SetKwInput{Output}{output}
	\Output{Flow map $\hat{X}$.}
\end{algorithm}

\begin{wrapfigure}[14]{r}{0.32\textwidth}
	\centering
	\includegraphics[width=0.30\textwidth]{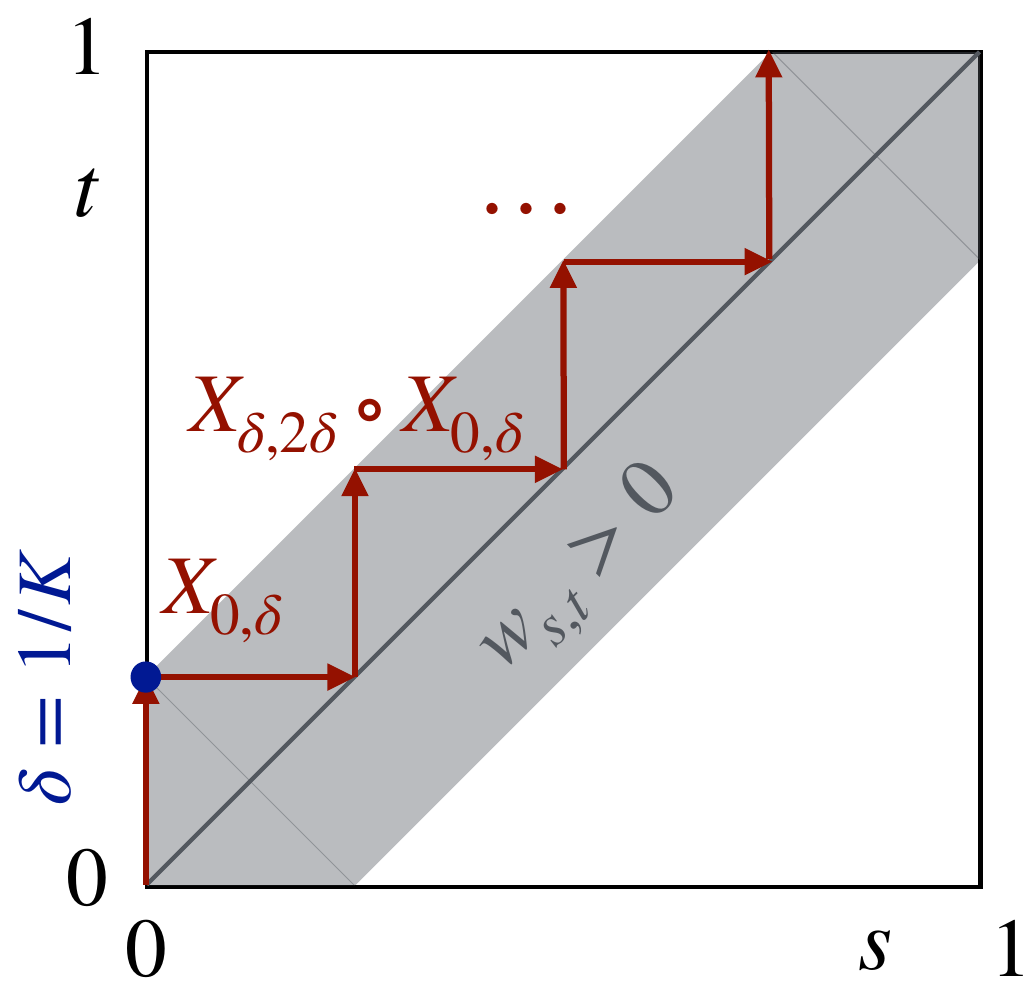}
	\caption{Schematic illustrating the weight $w_{s, t}$ in the FMM loss, which can be tuned to arrive at different learning schemes.}
	\label{fig:cartoon}
\end{wrapfigure}
In the loss~\eqref{eq:obj:match}, we are free to adjust the weight factor $w_{s,t}$, as illustrated in \cref{fig:cartoon}.
However, since we need to learn both the map $X_{s,t}$ and its inverse $X_{t,s}$, it is necessary to enforce the symmetry property $w_{t,s} = w_{s,t}$.
If we learn the map for all $(s,t) \in [0,1]^2$ using, for example, $w_{s,t} = 1$, then we can generate samples from~$\rho_1$ in one step via $X_{0,1}(x_0)$ with $x_0\sim\rho_0$.
If we learn the map in a strip, as shown in \cref{fig:cartoon}, the learning becomes simpler but we need to use multiple steps to generate samples from $\rho_1$; the number of steps then depends of the width of the strip.
We note that the second term enforcing invertibility in \cref{eq:obj:match} comes at no additional cost on the forward pass, because $\hat{X}_{s, t}(\hat{X}_{t, s}(I_t))$ can be computed at the same time as $\partial_t \hat{X}_{s, t}(\hat{X}_{t, s}(I_t))$ with standard Jacobian-vector product functionality in modern deep learning packages.
A summary of the flow map matching procedure is given in~\cref{alg:fmm}.

\subsection{Eulerian estimation or Eulerian distillation?}
\label{sec:eul:est}
In light of~\Cref{prop:fmm}, the reader may wonder whether we could also perform direct estimation in the Eulerian setup, using for example the objective
\begin{equation}
	\label{eqn:L_backward:2}
	\int_{[0,1]^2} w_{s,t} \E\Big[\big|\partial_s \hat{X}_{s, t}(I_s) + \dot I_s  \cdot \nabla \hat{X}_{s, t}(I_s)\big|^2\Big] dsdt.
\end{equation}
The above loss may be obtained from~\eqref{eqn:L_backward} by replacing the appearance of $b_s(I_s)$ with $\dot I_s$.
Unfortunately, \eqref{eqn:L_backward:2} is not equivalent to~\eqref{eqn:L_backward}. To see why, we can expand the expectation in \eqref{eqn:L_backward:2}:
\begin{equation}
	\label{eq:expand}
	\begin{aligned}
		 & \E\Big[\big|\partial_s \hat{X}_{s, t}(I_s) + \dot I_s  \cdot \nabla \hat{X}_{s, t}(I_s)\big|^2\Big]                                                                                                     \\
		 & = \E\Big[\big|\partial_s \hat{X}_{s, t}(I_s)\big|^2 + 2 (\dot I_s  \cdot \nabla \hat{X}_{s, t}(I_s))\cdot \partial_s \hat{X}_{s, t}(I_s) + \big|\dot I_s  \cdot \nabla \hat{X}_{s, t}(I_s)\big|^2\Big].
	\end{aligned}
\end{equation}
The cross term is linear in $\dot I_s$, so that we can use the tower property of the conditional expectation to see
\begin{equation}
	\label{eq:cterm}
	\begin{aligned}
		\E\big[(\dot I_s  \cdot \nabla \hat{X}_{s, t}(I_s))\cdot \partial_s \hat{X}_{s, t}(I_s)\big] & = \E\big[(\E[\dot I_s|I_s]  \cdot \nabla \hat{X}_{s, t}(I_s))\cdot \partial_s \hat{X}_{s, t}(I_s)\big], \\
		                                                                                             & = \E\big[(b_s(I_s)  \cdot \nabla \hat{X}_{s, t}(I_s))\cdot \partial_s \hat{X}_{s, t}(I_s)\big] .
	\end{aligned}
\end{equation}
However, the tower property cannot be applied to the last term in~\eqref{eq:expand} since it is quadratic in $\dot I_s$, i.e.
\begin{equation}
	\label{eq:lterm}
	\begin{aligned}
		\E\Big[ \big|\dot I_s  \cdot \nabla \hat{X}_{s, t}(I_s)\big|^2\Big]\not= \E\Big[ \big|b_s(I_s)  \cdot \nabla \hat{X}_{s, t}(I_s)\big|^2\Big].
	\end{aligned}
\end{equation}
Since this term depends on $\hat X$, it cannot be neglected in the minimization, and the minimizer of~\eqref{eqn:L_backward:2} is not the same as that of~\eqref{eqn:L_backward}.
Recognizing this difficulty, consistency models~\citep{song_consistency_2023, song_improved_2023, kim_consistency_2024} place a $\stopgrad{\cdot}$ on the term  $\dot I_s  \cdot \nabla \hat{X}_{s, t}(I_s)$ when computing the gradient of the loss~\eqref{eqn:L_backward:2}.
The resulting iterative scheme used to update $\hat X$ then has the correct fixed point at $\hat X = X$, as summarized in the following proposition.
\begin{restatable}[Eulerian estimation]{proposition}{emm}
	\label{prop:emm}
	The flow map $X_{s, t}$ is a critical point of the loss
	\begin{align}
		\label{eq:obj:eul:match}
		\calL_{\ee}(\hat{X}) = \int_{[0,1]^2} w_{s,t} \E \big [| \partial_s \hat X_{s,t} (I_s) - \stopgrad{\dot{I}_s \cdot \nabla \hat X_{s,t} (I_s)}|^2\big]  ds dt,
	\end{align}
	subject to the boundary condition $\hat{X}_{s, s}(x) = x$ for all $x \in \R^d$ and for all $s \in \R$.
	In~\eqref{eq:obj:eul:match}, $w_{s,t} > 0$ and $\E$ is taken over the coupling $(x_0, x_1) \sim \rho(x_0, x_1)$, $z \sim \mathsf N(0, \Id)$.
	The operator $\stopgrad{\cdot}$ indicates that the argument is ignored when computing the gradient.
\end{restatable}
We note that it is challenging to guarantee that the critical point described in~\cref{prop:emm} is stable and attractive for an iterative gradient-based scheme, as the objective function is not guaranteed to decrease due to the appearance of the $\stopgrad{\cdot}$ operator.
Nonetheless, some works have found success with related approaches on large-scale image generation problems \citep{song_consistency_2023}.

\begin{algorithm}[t]
	\caption{Progressive flow map matching (PFMM)}
	\label{alg:pfmm}
	\SetAlgoLined
	\SetKwInput{Input}{input}
	\Input{Interpolant coefficients $\alpha_t, \beta_t, \gamma_t$; weight $w_{s, t}$; $K$-step flow map $\hat{X}$; batch size $M$.}
	\Repeat{converged}{
		Draw batch $(s_i, t_i, I_{s_i})_{i=1}^M$ and compute $t_k^i = s_i + (k-1)(t_i - s_i)$ for $k = 1, \hdots, K$.\\
		Compute $\hat{\calL}_{\pfmm} = \frac{1}{M}\sum_{i=1}^M \left(|\check{X}_{s_i, t_i}(I_{s_i}) - \big(\hat{X}_{t_{K-1}^i, t_K^i}\circ \cdots \circ \hat{X}_{t_1^i, t_2^i}\big)(I_{s_i})|^2\right)$.\\
		Take gradient step on $\hat{\calL}_{\pfmm}$ to update $\check X$.
	}
	\SetKwInput{Output}{output}
	\Output{One-step flow map $\check {X}$.}
\end{algorithm}

\subsection{Progressive distillation}
\label{sec:progress}

Empirically, we found directly learning a one-step map to be challenging in practice.
Convergence was significantly improved  by taking $w_{s,t} = w_{t,s} = \mathbb{I}\left(|t-s| \leq 1/K\right)$ for some $K \in \mathbb{N}$ where $\mathbb{I}$ denotes an indicator function.
Given such a $K$-step model, it can be converted into a one-step model using a map distillation loss that is similar to progressive distillation~\citep{salimans_progressive_2022} and neural operator approaches~\citep{zheng_fast_2023}.

\begin{restatable}[Progressive flow map matching]{lemma}{fmmd}
	\label{prop:fmmd}
	Let $\hat X$ be a two-time flow map. Given $K\in \N$, let $t_k = s + \frac{k-1}{K-1}(t-s)$ for $k=1, \ldots, K$. Then the unique minimizer over $\check X$ of the objective
	\begin{align}
		\label{eq:obj:match:D}
		\calL_{\pfmm}(\check X) = \int_{[0,1]^2} w_{s,t}  \E \Big[\big| \check X_{s,t} (I_s) - \big(\hat{X}_{t_{K-1}, t_K}\circ \cdots \circ \hat{X}_{t_1, t_2}\big)(I_s)\big|^2\Big] ds dt,
	\end{align}
	produces the same output in one step as the $K$-step iterated map $\hat X$. Here $w_{s,t}> 0$, and $\E$ is taken over the coupling $(x_0, x_1) \sim \rho(x_0, x_1)$ and $z \sim \mathsf N(0, \Id)$.
\end{restatable}
We note that $\hat X$ is fixed in~\eqref{eq:obj:match:D} and serves as the teacher, so we only need to compute the gradient with respect to the parameters of $\check X$.
In practice, we may train $\hat{X}$ using~\eqref{eq:obj:match} over a class of neural networks and then freeze its parameters.
We may then use~\eqref{eq:obj:match:D} to distill $\hat{X}$ into a more efficient model $\check{X}$, which can be initialized from the parameters of $\hat{X}$ for an efficient warm start.
Alternatively, we can add \eqref{eq:obj:match:D} to \eqref{eq:obj:match} and train a single network $\hat X$ by setting $\check X = \hat X$ in \eqref{eq:obj:match:D} and placing a $\stopgrad{\cdot}$ on the multistep term.

If $K$ evaluations of $\hat{X}$ are expensive, we may iteratively minimize~\eqref{eq:obj:match:D} with some number $M < K$ evaluations of $\hat{X}$ and then replace $\hat{X}$ by $\check{X}$, similar to progressive distillation~\citep{salimans_progressive_2022}.
For example, we may take $M = 2$ and then minimize~\eqref{eq:obj:match:D} $\ceil{\log_2 K}$ times to obtain a one-step map.
Alternatively, we can first generate a dataset of $(s, t, I_s, (\hat{X}_{t_{K-1}, t_K}\circ \cdots \circ \hat{X}_{t_1, t_2})(I_s))$ in a parallel offline phase, which converts~\eqref{eq:obj:match:D} into a simple least-squares problem.
Finally, if we are only interested in using the map forward in time, we can set $w_{s,t}=1$ if $s\le t$ and $w_{s,t}=0$ otherwise.
The resulting procedure is summarized in~\cref{alg:pfmm}.

\section{Numerical Realizations}
\label{sec:exp}
In this section, we study the efficacy of the four methods introduced in~\cref{sec:theory}:
the Lagrangian map distillation discussed in~\cref{prop:distill}, the Eulerian map distillation discussed in~\cref{prop:backward_loss}, the direct training approach of~\cref{prop:fmm}, and the progressive flow map matching approach of~\cref{prop:fmmd}.
We consider their performance on a two-dimensional checkerboard dataset, as well as in the high-dimensional setting of image generation, to highlight differences in their training efficiency and performance.

To ensure that the boundary conditions on the flow map $\hat X_{s,t}$ defined in~\eqref{eqn:flow_map_dyn:L} are enforced, in all experiments, we parameterize the map using the ansatz
\begin{equation}
	\label{eq:map:anstaz}
	X_{s,t}(x) = x + (t-s) v^\theta_{s,t}(x),
\end{equation}
where $v^\theta_{s,t}(x) : [0,T]^2 \times \mathbb R^d \rightarrow \mathbb R^d$ is a neural network with parameters~$\theta$.

\subsection{2D Illustration}
\begin{figure}[t]
	\centering
	\includegraphics[width=1.0\linewidth]{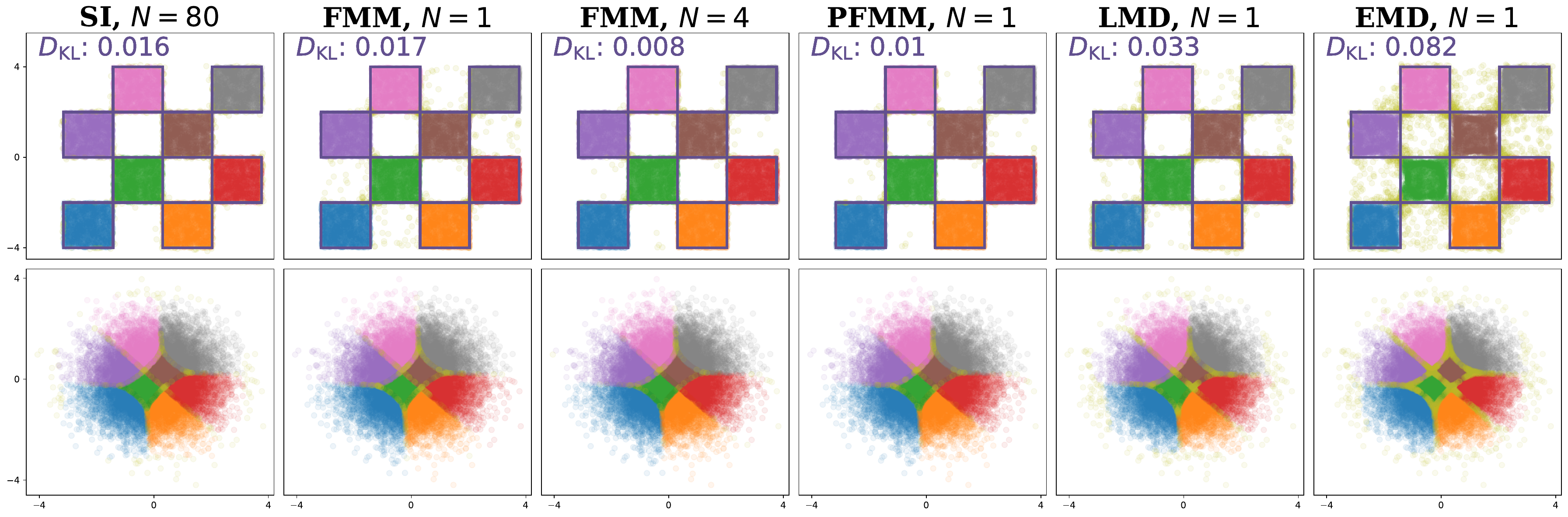}
	\caption{\textbf{Two-dimensional results.}
	Comparison of the various map-matching procedures on the 2D checkerboard dataset, with the results from the probability flow ODE of a stochastic interpolant integrated using $N=80$ discretization steps as reference (top, first panel from left).
	The one-step map obtained by FMM when learning on $(s,t)=[0,1]^2$ (top, second panel) performs as well as SI. Moreover, the accuracy improves if we allow four steps instead of one (top, third panel).
	This four-step map can be accurately distilled into a one-step map via PFMM (top, fourth panel).
	The one-step map obtained by distilling the pre-trained $b$ via LMD (top, fifth panel) performs reasonably well too, and is better than the one-step map obtained by distilling the same $b$ via EMD (top, sixth panel).
	A KL-divergence between each model distribution and the target is provided to quantify performance, indicating that FMM, its progressive distillation, and LMD are closest to the probability flow ODE baseline. The bottom row indicates, by color, how points from the Gaussian base are assigned by each of the respective maps. The yellow dots are points that mistakenly land outside the checkerboard. These results indicate that the primary source of error in each case is handling the discontinuity of the optimal map at the edges of the checker. See Appendix \ref{app:exp:checker} for more details.}
	\label{fig:checker}
\end{figure}
As a simple illustration of our method, we consider learning the flow map connecting a two-dimensional Gaussian distribution to the checkerboard distribution presented in~\cref{fig:checker}. Note that this example is challenging because the target density is supported on a compact set, and because the target is discontinuous at the edge of this set.
This mapping can be achieved, as discussed in~\cref{sec:theory}, in various ways: (a) implicitly, by solving~\eqref{eqn:ode} with a learned velocity field using stochastic interpolants (or a diffusion model), (b) directly, using the flow map matching objective in~\eqref{eq:obj:match}, (c) progressively matching the flow map using \eqref{eq:obj:match:D}, or (d/e) distilling the map using the Eulerian \eqref{eqn:L_backward} or Lagrangian \eqref{eqn:distillation_loss} losses.
In each case, we use a fully connected neural network with $512$ neurons per hidden layer and $6$ layers to parameterize either a velocity field $\hat b_t(x)$ or a flow map $\hat X_{s,t}(x)$.
We optimize each loss using the Adam~\citep{kingma2017adam} optimizer for $5\times 10^4$ training iterations.
The results are presented in~\cref{fig:checker}, where we observe that using the one-step $\hat X_{0,1}(x)$ directly learned by minimizing~\eqref{eq:obj:match} over the entire interval $(s, t) \in [0, 1]^2$ performs worse than learning with $|t-s| < 0.25$ and sampling with 4 steps.
With this in mind, we use the 4-step map as a teacher to minimize the PFMM loss, which produces a high-performing one-step map.
We also note that the EMD loss performs worse than the LMD loss when distilling the map from a learned velocity field.

\subsection{Image Generation}

\begin{figure}[b!]
	\centering
	\begin{overpic}[width=0.48\textwidth]{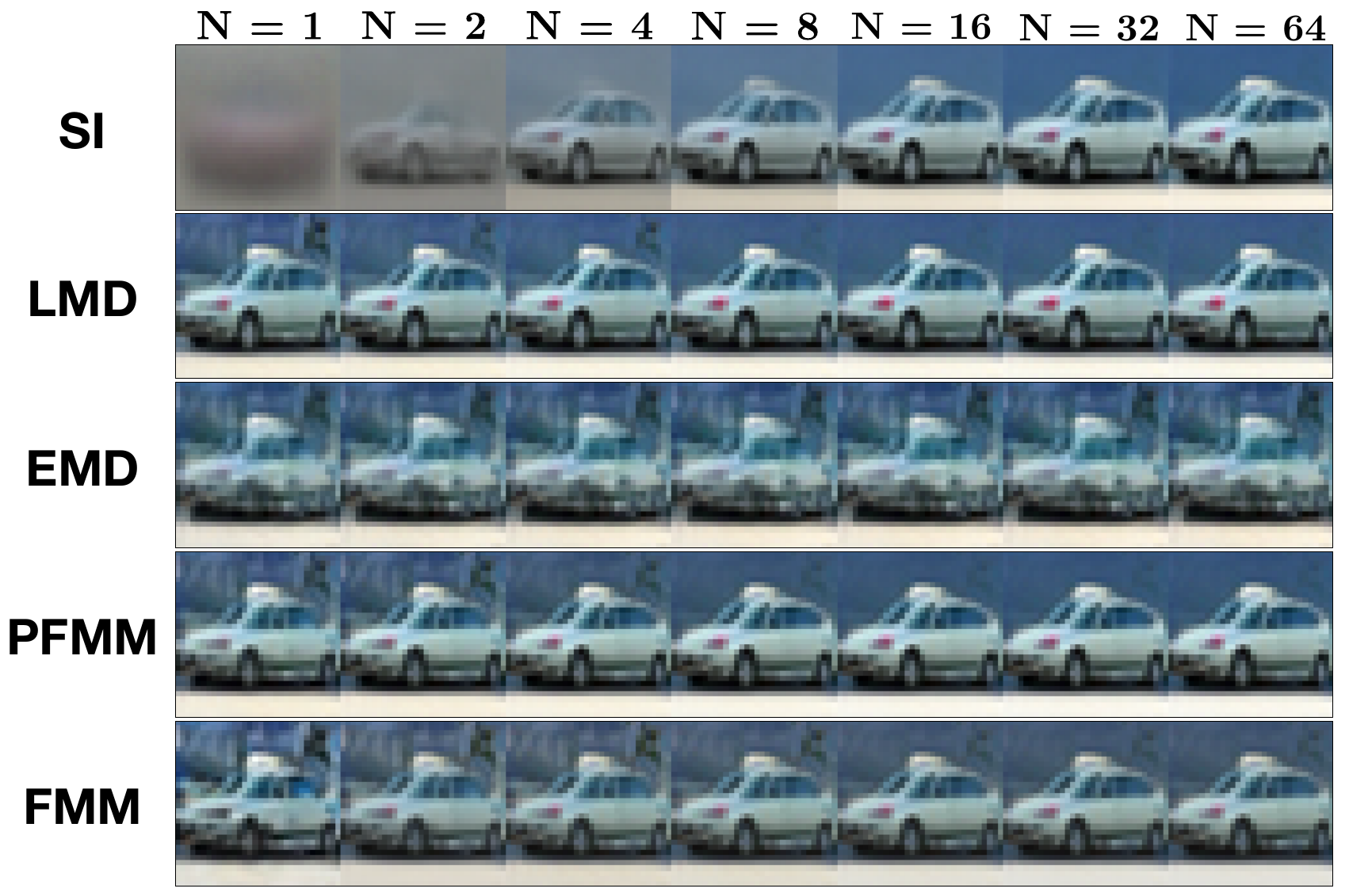}
		\centering
		\put(0,61){\textbf{A}}
	\end{overpic}
	\begin{overpic}[width=0.51\textwidth]{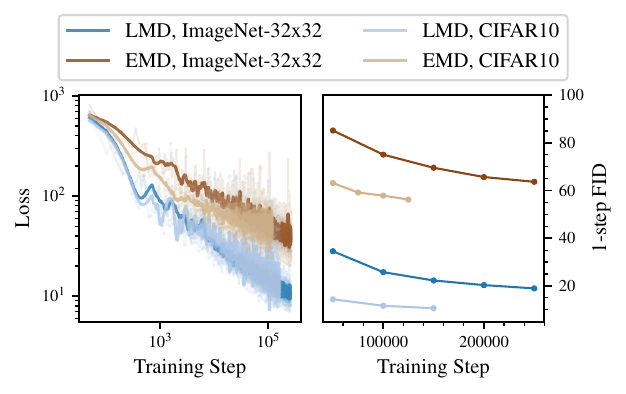}
		\centering
		\put(2,58){\textbf{B}}
	\end{overpic}
	\caption{%
		(A) Qualitative comparison between the standard stochastic interpolant approach (SI), Lagrangian map distillation (LMD), Eulerian map distillation (EMD), and progressive flow map matching (PFMM).
		SI produces good images for a sufficiently large number of steps, but performs poorly for few steps.
		LMD performs well in the very-few step regime, and outperforms EMD significantly.
		PFMM performs well at any number of steps, though performs slightly worse than LMD in the very-few step regime.
		(B) Quantitative comparison between EMD and LMD on both CIFAR-10 and ImageNet $32\times32$.
		Despite both having the same minimizer, LMD trains faster, and attains a lower loss value and a lower FID for a fixed number of training steps.
	}
	\label{fig:performance_comp}
\end{figure}

Motivated by the above results, we consider a series of image generation experiments on the CIFAR-10 and ImageNet-$32\times32$ datasets.
For comparison, we benchmark the method against alternative techniques that seek to lower the number of steps needed to produce samples with stochastic interpolant models, e.g. by straightening the ODE trajectories using minibatch OT \citep{pooladian2023multisample, tong_improving_2023}.
We train all of our models from scratch, so as to control the design space of the comparison.
For clarity, we label when benchmark numbers are quoted from the literature.
\begin{table}[t!]
	\centering
	\begin{tabular}{lcccccccc}
		\toprule
		\multirow{2}{*}{Method} & \multicolumn{2}{c}{N=2} & \multicolumn{2}{c}{N=4} & \multirow{2}{*}{Baseline}                \\
		\cmidrule(lr){2-3} \cmidrule(lr){4-5}
		                        & FID                     & T-FID                   & FID                       & T-FID &      \\
		\midrule
		SI                      & 112.42                  & -                       & 34.84                     & -     & 5.53 \\
		EMD                     & 48.32                   & 34.19                   & 44.35                     & 30.74 & 5.53 \\
		LMD                     & 7.13                    & 1.27                    & 6.04                      & 1.05  & 5.53 \\
		PFMM                    & 18.35                   & 7.02                    & 11.14                     & 1.52  & 8.44 \\
		\bottomrule
	\end{tabular}
	\caption{Comparison of various distillation methods using FID and Teacher-FID metrics on the CIFAR-10 dataset. Note that for PFMM, no velocity model (e.g. from a stochastic interpolant) is needed. It relies solely on the minimization of \eqref{eq:obj:match} and \eqref{eq:obj:match:D}.
		Baseline indicates the FID of the teacher model (a velocity field for EMD and LMD integrated with
		an adaptive fifth-order Runge-Kutta scheme, and a flow map for PFMM) against the true data.
	}
	\label{tab:comparison}
\end{table}

For learning of the flow map, we use a U-Net architecture following~\citep{dhariwal_diffusion_2021}. For LMD and EMD that require a pre-trained velocity field to distill, we also use a U-Net architecture for $b$.
Because the flow map $X_{s, t}$ is a function of two times, we modify the architecture.
Both $s$ and $t$ are embedded identically to $t$ in the original architecture.
The result is concatenated and treated like $t$ in the original architecture for downstream layers.
We benchmark the performance of the methods using the Frechet Inception Distance (FID), which computes a measure of similarity between real images from the dataset and those generated by our models.
In addition, we compute what we denote as the Teacher-FID (T-FID).
This metric computes the same measure of similarity, but now between images generated by the teacher model and those generated by the distilled model, rather than leveraging the original dataset.
This measure allows us to directly benchmark the convergence of the distillation method, as it captures discrepancies between the distribution of samples generated by the teacher and the distribution of samples generated by the student.
In addition, this allows us to benchmark accuracy independent of the overall performance of the teacher, as our teacher models were trained with limited compute.

\paragraph{Sampling efficiency}
In Table \ref{tab:comparison}, we compute the FID and T-FID for the stochastic interpolant, Eulerian, Lagrangian, and progressive distillation models on 2 and 4-step generation for CIFAR-10.
The stochastic interpolant was trained to a baseline FID (sampling with an adaptive solver) of 5.53, and was used as the teacher for EMD and LMD.
The teacher for PFMM was an FMM model trained with $|t-s| < 0.25$ to an FID of 8.44 using 8-step sampling. We observe that LMD and EMD methods can effectively distill their teachers and obtain low T-FID scores.
In addition, the 2 and 4-step samples from these methods far outperform the stochastic interpolant. This sampling efficiency is also apparent in the left side of~\cref{fig:performance_comp}, in which with just 1 to 4 steps, the LMD and PFMM methods can produce effective samples, particularly when compared to the flow matching approach.

\begin{wraptable}[14]{r}{0.5\textwidth} 
	\vspace{2.5mm}
	\centering
	\begin{tabular}{cccc}
		\toprule
		$\mathbf{N}$ & \textbf{DDPM} & \textbf{BatchOT} & \textbf{FMM (Ours)} \\
		\midrule
		20           & 63.08         & \textbf{7.71}    & 9.68                \\
		8            & 232.97        & 15.64            & \textbf{12.61}      \\
		6            & 275.28        & 22.08            & \textbf{14.48}      \\
		4            & 362.37        & 38.86            & \textbf{16.90}      \\
		\bottomrule
	\end{tabular}
	\caption{FID scaling with number of function evaluations $N$ to produce a sample on ImageNet-$32\times32$.
		We show a comparison between DDPM~\citep{ho2020denoising} and multi-sample Flow Matching using the BatchOT method \citep{pooladian2023multisample} to flow map matching.
		The first two columns are quoted from \citet{pooladian2023multisample}.
		Note that no distillation is used here, but rather direct minimization of \eqref{eq:obj:match}, using $|t-s| < 0.25$.
	}
	\label{tab:scaling}
\end{wraptable}
Without any distillation, FMM can also produce effective few-step maps.
Training an FMM model on the ImageNet-$32\times32$ dataset, we observe (Table \ref{tab:scaling}) that FMM achieves much better few-step FID when compared to denoising diffusion models (DDPM), and better FID than mini-batch OT interpolant methods \citep{pooladian2023multisample}.
In the higher-step regime, the interpolant methods perform marginally better.

\paragraph{Eulerian vs Lagrangian distillation}
Remarkably, we find a stark performance gap between the Eulerian and Lagrangian distillation schemes.
This is evident in both parts of~\cref{fig:performance_comp}, where we see that higher-step sampling with EMD only marginally improves image quality, and where the LMD loss for both CIFAR10 and ImageNet-$32\times32$ converges an order of magnitude faster than the EMD loss.
The same holds for FIDs over training, given in the right-most plot in the figure.
Note that both LMD and EMD loss functions have a global minimum at 0, so that the loss plots suggest continued training will improve distillation quality, but at very different rates.

\subsubsection{Efficient Style Transfer}
To illustrate some of the downstream tasks facilitated by our approach, we now describe a means of performing style transfer with the two-time invertible flow map that we call ``consistency style transfer''.
A class conditional sample $x_1$ with class label $y$ can be partially inverted to an earlier time $s'$ via $X_{1,s'}(x_1; y)$ for $s' < 1$.
From here, replacing the class label with $y' \neq y$, we can resample the flow map $X_{s', 1}(X_{1, s'}(x_1; y); y')$ to sample the conditional distribution associated to $y'$ while maintaining the style of the original class.
We demonstrate this principle in Figure \ref{fig:style}, which shows that maps discretized with $n=8$ step can be used convert between classes.
\begin{figure}[t]
	\centering
	\vspace{-20pt} 
	\includegraphics[width=0.5\linewidth]{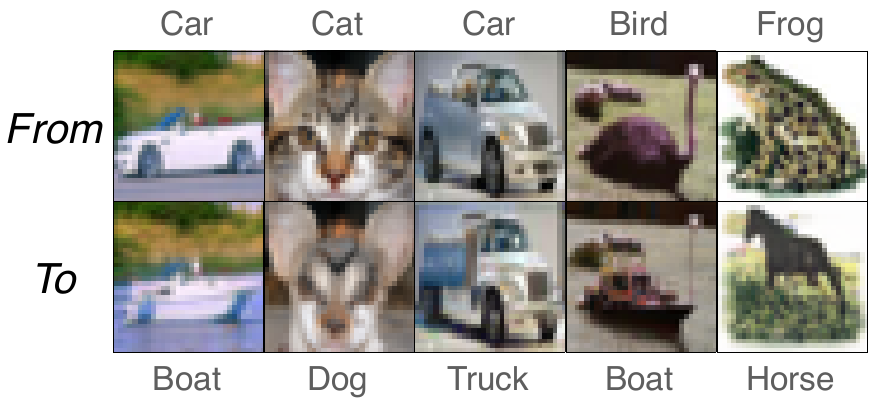}
	\caption{Consistency style transfer on CIFAR-10. Original class-conditional images from the dataset (top row) are pushed backward in time to $X_{1, s'=0.3}(x_1; y)$ and then pushed forward using a new class. The original styles of the images are maintained, while their subjects are replaced with those from the new class (bottom row).}
	\label{fig:style}
\end{figure}
This serves both as verification of the cycle consistency as well as an illustration of the potential applications of the two-time flow map.
In practice, the extent of the preserved style depends on how far back in time towards the Gaussian base at $s'=0$ we push the original sample.
Here, we use $s'=0.3$.

\section{Conclusion}
\label{sec:conc}
In this work, we developed a framework for learning the two-time flow map for generative modeling: either by distilling a pre-trained velocity model with the Eulerian or Lagrangian losses, or by directly training in the stochastic interpolant framework.
We empirically observe that while using more steps with the learned map improves sample quality, a substantially lower number is needed when compared to other generative models built on dynamical transport.
Future work will investigate how to improve the training and the neural network architecture so as to further reduce the number of steps without sacrificing accuracy, and to improve convergence for direct training of one-step maps.


\bibliographystyle{unsrtnat}
\bibliography{paper}

\newpage
\appendix
\section{Probability flow ODEs from stochastic interpolants and score-based diffusion models}
\label{app:si_diff}
For the reader's convenience, we now recall how to construct probability flow ODEs using either flow matching with stochastic interpolants or score-based diffusion models. We also recall the connection between these two formalisms.

\subsection{Transport equation}
\label{app:transp:meas}

Generative models based on probability flow ODEs leverage the property that solutions to~\eqref{eqn:ode} push forward their initial conditions onto samples from the target:
\begin{restatable}[Transport equation]{lemma}{transp}
	Let $\rho_t = \Law(x_t)$ be the PDF of the solution to~\eqref{eqn:ode} assuming that $x_0\sim \rho_0$. Then $\rho_t$ satisfies\begin{equation}
		\label{eq:transport}
		\partial_t \rho_t(x) + \nabla \cdot (b_t(x) \rho_t(x)) = 0, \qquad \rho_{t=0}(x)=\rho_0(x)
	\end{equation}
	where $\nabla$ denotes a gradient with respect to $x$.
\end{restatable}
The interest of this result is that it can be used in reverse: if we show that a PDF $\rho_t$ satisfies the transport equation~\eqref{eq:transport}, then in practice we can use the probability flow ODE~\eqref{eqn:ode} to sample $\rho_t$ at any time $t>0$, so long as we can sample initial conditions from $\rho_0$.

\begin{proof}
	The proof proceeds via the weak formulation of~\cref{eq:transport}.
	Let $\phi\in C^1_b(\R^d) $ denote an arbitrary test function.
	By definition,
	\begin{equation}
		\label{def:exp:2}
		\forall t \in [0,1]\quad : \quad \int_{\R^d} \phi(x) \rho_t(x) dx =  \E[ \phi(x_t)]
	\end{equation}
	where $x_t$ is given by~\eqref{eq:stoch:interp} and where the expectation on the right hand-side is taken over the law of the initial conditions $x_{t=0}=x_0\sim\rho_0$.
	Taking the time derivative of this equality, we deduce that
	\begin{equation}
		\label{def:exp:diff2}
		\begin{aligned}
			\int_{\R^d} \phi(x) \partial_t \rho_t(x) dx & =  \E[\dot x_t \cdot \nabla \phi(x_t) ]                &  & \text{by the chain rule}             \\
			                                            & = \E[b_t(x_t) \cdot \nabla \phi(x_t) ]                 &  & \text{using the ODE~\eqref{eqn:ode}}
			\\
			                                            & = \int_{\R^d} b_t(x) \cdot \nabla \phi(x) \rho_t(x) dx &  & \text{by definition of $\rho_t(x)$}
		\end{aligned}
	\end{equation}
	This is the weak form of the transport equation~\eqref{eq:transport}.
	It can be written as~\eqref{eq:transport}, since it admits strong solutions by our assumptions on $b_t(x)$.
\end{proof}

\subsection{Stochastic interpolants and probability flows}
\label{secsi}
The stochastic interpolant framework offers a simple and versatile paradigm to construct generative models.

\sidef*

\pflow*

\begin{proof}
	Let $\phi\in C^1_b(\R^d) $ denote an arbitrary test function.
	By definition,
	\begin{equation}
		\label{def:exp:I}
		\forall t \in [0,1]\quad : \quad \int_{\R^d} \phi(x) \rho_t(x) dx =  \E[ \phi(I_t)]
	\end{equation}
	where $I_t$ is given by~\eqref{eq:stoch:interp} and where the expectation on the right-hand side is taken over the coupling $(x_0, x_1) \sim \rho(x_0, x_1)$ and $z \sim \mathsf N(0, I)$.
	Taking the time derivative of this equality, we deduce that
	\begin{equation}
		\label{def:exp:diffI}
		\begin{aligned}
			\int_{\R^d} \phi(x) \partial_t \rho_t(x) dx & =  \E[\dot I_t \cdot \nabla \phi(I_t) ]                &  & \text{by the chain rule}                                    \\
			                                            & = \E[\E[\dot I_t|I_t] \cdot \nabla \phi(I_t) ]         &  & \text{by the tower property of the conditional expectation} \\
			                                            & = \E[b_t(I_t) \cdot \nabla \phi(I_t) ]                 &  & \text{by definition of $b_t(x)$ in \eqref{eqn:b:def}}
			\\
			                                            & = \int_{\R^d} b_t(x) \cdot \nabla \phi(x) \rho_t(x) dx &  & \text{by definition of $\rho_t(x)$}
		\end{aligned}
	\end{equation}
	This is the weak form of the transport equation~\eqref{eq:transport}.
\end{proof}

This result implies that the PDF $\rho_t$ of the stochastic interpolant $I_t$  is also the PDF of the solution $x_t$ of the probability flow ODE~\eqref{eqn:ode} with the velocity field $b_t(x)$ defined in~\eqref{eqn:b:def}.
Algorithmically, this means that we can use the associated flow map as a generative model.

\subsection{Score-based diffusion models}
\label{sec:sbdm}

For simplicity, we focus on variance-preserving diffusions; extensions to more general classes of diffusions are straightforward.
These models are based on variants of the Ornstein-Uhlenbeck process, defined as the solution to the stochastic differential equation (SDE):
\begin{equation}
	dX_t = -X_t dt +\sqrt{2}\,dW_t, \qquad X_{t=0} = a\sim \rho_*,
	\label{sde1}
\end{equation}
where $W_t$ is a Wiener process.
The solution to \eqref{sde1} for the initial condition $X_{t=0}=a$ is:
\begin{equation}
	X_t = a\,e^{-t}+\sqrt{2}\int_0^t  e^{-t+t'}dW_{t'}.
	\label{sde2}
\end{equation}
By the It\^o Isometry, the second term on the right-hand side is a Gaussian process with mean zero and covariance given by
\begin{equation}
	\E\left[\left(\sqrt{2}\int_0^te^{-t+t'}dW_{t'}\right)^2\right]= 2\Id \int_0^te^{-2t+2t'}dt' = (1-e^{-2t})\Id.
\end{equation}
Therefore, at any time $t$, the law of~\eqref{sde2} is that of a Gaussian random variable with mean $x e^{-t}$ and covariance $(1-e^{-2t})\Id$.
That is, it can be represented as
\begin{equation}
	X_t \stackrel{d}{=} a\,e^{-t}+\sqrt{1-e^{-2t}} z,
	\label{sde3}
\end{equation}
where $a\sim \rho_*$, $z\sim  N(0,\Id)$, and $a\perp z$.
\eqref{sde3} shows that the law of $X_t$ converges exponentially fast to that of the standard normal variable as $t\to\infty$ for all initial conditions~$a$.

The key insight of score-based diffusion models is to time-reverse the SDE \eqref{sde1} to obtain a process that turns Gaussian samples into samples from the target.
Alternatively, this time-reversal can be performed using the associated probability flow ODE.
To do so, we can start from the evolution equation for the PDF of the solution to the SDE~\eqref{sde1}. Denoting this PDF by $\tilde \rho_t$,  it satisfies the Fokker-Planck equation (FPE)
\begin{equation}
	\label{eqn:vp_diff_fpe}
	\partial_t \tilde\rho_t = \nabla\cdot(x\tilde\rho_t) +\Delta\tilde\rho_t, \qquad \tilde \rho_0 = \rho_*.
\end{equation}
To time-reverse the solution to~\eqref{eqn:vp_diff_fpe} we define $\rho_t = \tilde \rho_{T-t}$ for some large $T>0$, and derive an equation for $\rho_t$ from~\eqref{eqn:vp_diff_fpe}:
\begin{equation}
	\begin{aligned}
		\partial_t \tilde\rho_t & = -\partial_t \tilde \rho_{T-t}                                       \\
		                        & = - \nabla\cdot(x \tilde\rho_{T-t}) - \Delta \tilde \rho_{T-t}        \\
		                        & = - \nabla\cdot([x + \nabla \log \tilde \rho_{T-t} ]\tilde\rho_{T-t}) \\
		                        & = - \nabla\cdot([x + \nabla \log \rho_{t} ]\rho_{t})
		\label{eq:fp:rev}
	\end{aligned}
\end{equation}
where we used the identity $\Delta  \tilde \rho_{T-t} = \nabla \cdot( \tilde \rho_{T-t} \nabla \log \tilde \rho_{T-t})$. Equation~\eqref{eq:fp:rev} can be written as
\begin{equation}
	\partial_t \tilde\rho_t + \nabla\cdot([x + s_t(x)] \rho_t) =0, \qquad \rho_{t=0} = \tilde \rho_{t=T} \approx N(0,\Id)
	\label{eq:fp:transp}
\end{equation}
where $s_t(x) = \nabla \log \rho_t(x) = \nabla \log \tilde \rho_{T-t}(x)$ is the score.
This quantity can be estimated by regression using the solutions to~\eqref{sde1}, which are explicitly available via~\eqref{sde3}.
Equation~\eqref{eq:fp:transp} is in the form of~\eqref{eq:transport} with a velocity field $b_t(x)$ given by
\begin{equation}
	\label{eq:pflow:sb}
	b_t(x) = x + s_t(x).
\end{equation}
Flow map matching can therefore be used to distill score-based diffusion models using this velocity field.

\subsection{Connections}
\label{app:connect}

Here we show how score-based diffusion models can be related to stochastic interpolants.
Recall that the law of the solution to the SDE~\eqref{sde1} is given by~\eqref{sde3}.
If we change time according to $t \mapsto - \log t$, we map $[0,\infty)$ onto $[1,0)$ and we arrive at
\begin{equation}
	X_{-\log t} \stackrel{d}{=} I_t =  a t+\sqrt{1-t^2} \, z, \qquad  a\sim \rho_*, \ \  z\sim  N(0,\Id), \ \text{ and } \ a\perp z.
	\label{sde4}
\end{equation}
This is a valid stochastic interpolant if we set $a=x_1$, $\beta_t=t$, $\alpha_t = 0$, and $\gamma_t = \sqrt{1-t^2} $ in~\eqref{eq:stoch:interp}\footnote{Note that $\gamma_0 =1$ here, consistent with the fact that the base PDF is Gaussian, so that $x_0$ and $z$ can be lumped into a single Gaussian variable}.
Hence, we have shown that the probability flow of the score-based diffusion model with velocity field~\eqref{eq:pflow:sb} can also be studied by using the stochastic interpolant~\eqref{sde4}.
Note that this reformulation has the advantage that it eliminates the small bias incurred by the finite value of $T$ in score-based diffusion models, since the stochastic interpolant~\eqref{sde4} transports the Gaussian random variable $z$ at  time $t=0$ onto the data point $a$ at time $t=1$ exactly.

\section{Proofs for Section \ref{sec:theory}}
\label{app:theory}

\flowmap*
\begin{proof}
	Taking the derivative with respect to $t$ of $X_{s,t}(x_s) = x_t$, we deduce
	\begin{equation}
		\label{eq:Lag:1}
		\partial_t X_{s,t}(x_s) = \dot x_t = b_t(x_t) = b_t(X_{s,t}(x_s))
	\end{equation}
	where we used the ODE~\eqref{eqn:ode} to obtain the second equality.
	Conversely, since the solutions to~\eqref{eqn:ode} and~\eqref{eqn:flow_map_dyn:L} from a given initial condition are unique, if we solve~\cref{eqn:flow_map_dyn:L} on $[s,t]$ with the condition $X_{s,s}(x_s)=x_s$ we must have $X_{s,t}(x_s) = x_t$ for all $(s,t)\in[0,1]^2$. Evaluating this expression at $x_s=x$ gives~\eqref{eqn:flow_map_dyn:L}. Also, for all $(s,t,t)\in [0,T]^3$, we have
	\begin{equation}
		\label{eq:Lag:2}
		X_{t,t}(X_{s,t}(x_s)) =X_{t,t}(x_t) = x_t = X_{s,t}(x_s).
	\end{equation}
	Evaluating this expression at $x_s=x$ gives~\eqref{eq:semigroup}.
\end{proof}

\lagdistill*
\begin{proof}
	Equation~\eqref{eqn:distillation_loss} is a physics-informed neural network (PINN)~\citep{raissi_physics-informed_2019} loss that is minimized only when the integrand is zero, i.e., when~\eqref{eqn:flow_map_dyn:L} holds.
\end{proof}

\euleq*
\begin{proof}
	Taking the derivative with respect to $s$ of  $X_{s,t}(X_{t,s}(x)) = x$ and using the chain rule, we deduce that
	\begin{equation}
		\label{eq:Eul:1}
		\begin{aligned}
			0 = \frac{d}{ds} X_{s,t}(X_{t,s}(x)) & =\partial_s X_{s,t}(X_{t,s}(x)) + \partial_s X_{t,s}(x) \cdot \nabla X_{s,t}(X_{t,s}(x)) \\
			                                     & =\partial_s X_{s,t}(X_{t,s}(x)) + b_s( X_{t,s}(x)) \cdot \nabla X_{s,t}(X_{t,s}(x))
		\end{aligned}
	\end{equation}
	where we used~\cref{eqn:flow_map_dyn:L} to get the last equality. Evaluating this expression at $X_{t,s}(x) = y$, then changing $y$ into $x$, gives~\cref{eqn:backward}.
\end{proof}

\euldistill*
\begin{proof}
	Equation~\eqref{eqn:L_backward} is a PINN loss that is minimized only when the integrand is zero, i.e, when~\eqref{eqn:backward} holds.
\end{proof}

\lagrangianerr*
\begin{proof}
	First observe that, by the one-sided Lipschitz condition~\eqref{eq:one:sided:lip},
	\begin{equation}
		\label{eq:no:spread}
		\begin{aligned}
			\partial_t |X_{s,t}(x) - X_{s,t}(y)|^2 & = 2 (X_{s,t}(x)-X_{s,t}(y))\cdot (b_t(X_{s,t}(x))-b_t(X_{s,t}(y))), \\
			                                       & \le 2C_t |X_{s,t}(x) - X_{s,t}(y)|^2.
		\end{aligned}
	\end{equation}
	Equation~\eqref{eq:no:spread} highlights that~\eqref{eq:one:sided:lip} gives a bound on the spread of trajectories.
	We note that we allow for $C_t<0$, which corresponds to globally contracting maps.
	Given~\eqref{eq:no:spread}, we now define
	\begin{equation}
		\label{eq:error:diag}
		E_{s,t} = \E\big[\big|X_{s, t}(I_s) - \hat{X}_{s, t}(I_s)\big|^2\big],
	\end{equation}
	where we recall that $X_{s,t}(x)$ satisfies $\partial_t X_{s,t}(x) = b_t(X_{s,t}(x))$ and $X_{s,s}(x) = x$.
	Taking the derivative with respect to $t$ of~\eqref{eq:error:diag}, we deduce
	\begin{equation}
		\label{eq:bound:err:1}
		\begin{aligned}
			\partial_t E_{s,t} & = 2\E\big[\big(X_{s, t}(I_s) - \hat{X}_{s, t}(I_s)\big)\cdot\big(b_t(X_{s, t}(I_s)) - \partial_t \hat{X}_{s, t}(I_s)\big)\big],                               \\
			                   & = 2 \E\big[\big(X_{s, t}(I_s) - \hat{X}_{s, t}(I_s)\big)\cdot\big(b_t(\hat X_{s, t}(I_s)) - \partial_t \hat{X}_{s, t}(I_s)\big)\big]                          \\
			                   & \quad + 2\E\big[\big(X_{s, t}(I_s) - \hat{X}_{s, t}(I_s)\big)\cdot\big(b_t(X_{s, t}(I_s))-b_t(\hat X_{s, t}(I_s)) \big)\big],                                 \\
			                   & \leq  \E\big[\big|X_{s, t}(I_s) - \hat{X}_{s, t}(I_s)\big|^2\big] +  \E\big[\big|b_t(\hat X_{s, t}(I_s)) - \partial_t \hat{X}_{s, t}(I_s)\big|^2\big]         \\
			                   & \quad +  2\E\big[\big(X_{s, t}(I_s) - \hat{X}_{s, t}(I_s)\big)\cdot\big(b_t(X_{s, t}(I_s))-b_t(\hat X_{s, t}(I_s)) \big)\big],                                \\
			                   & \equiv E_{s,t} + \delta^{\lmd}_{s,t} +  2\E\big[\big(X_{s, t}(I_s) - \hat{X}_{s, t}(I_s)\big)\cdot\big(b_t(X_{s, t}(I_s))-b_t(\hat X_{s, t}(I_s)) \big)\big].
		\end{aligned}
	\end{equation}
	Above, we defined the two-time Lagrangian distillation error,
	\begin{equation}
		\delta^{\lmd}_{s,t} = \E\big[\big|b_t(\hat X_{s, t}(I_s)) - \partial_t \hat{X}_{s, t}(I_s)\big|^2\big].
	\end{equation}
	By definition, the LMD loss can be expressed as  $L_{\lmd}(\hat X) = \int_{[0,T]^2} w_{s,t} \delta^{\lmd}_{s,t} ds dt$.
	Using \eqref{eq:one:sided:lip} in \eqref{eq:bound:err:1}, we obtain the relation
	\begin{equation}
		\begin{aligned}
			 & \partial_t E_{s,t} \le  (1+2C_t)E_{s,t} + \delta_{s,t}^{\lmd},
		\end{aligned}
	\end{equation}
	which implies that
	\begin{equation}
		\begin{aligned}
			\partial_t \big( e^{-t-2\int_s^t C_u du}E_{s,t} \big) \le  e^{-t-2\int_s^t C_u du} \delta_{s,t}^{\lmd}.
		\end{aligned}
	\end{equation}
	Since $E_{s,s} = 0$ this implies that
	\begin{equation}
		E_{s,t} \le \int_s^t e^{(t-u)+2\int_u^t C_t dt} \delta^{\lmd}_{s,u} du \le e^{t+2\int_s^t |C_t| dt} \int_s^t \delta^{\lmd}_{s,u} du.
	\end{equation}
	Above, we used that $(t, u) \in [0, t]^2$ so that $(t - u) \leq t$.
	This bound shows that $E_{0,1} \le e^{1+2\int_0^1 |C_t| dt} \int_0^1 \delta^{\lmd}_{0,u} du$, which can be written explicitly as (using $t$ instead of $u$ as dummy integration variable)
	\begin{equation}
		\E\big[\big|X_{0, 1}(x_0) - \hat{X}_{0, 1}(x_0)\big|^2\big] \le e^{1+2\int_0^1 |C_t| dt} \int_0^1 \E\big[\big|b_t(\hat X_{0, t}(x_0)) - \partial_t \hat{X}_{0, t}(x_0)\big|^2\big]dt,
	\end{equation}
	Now, observe that by definition,
	\begin{equation}
		\label{eqn:wasserstein_bound}
		W_2^2(\rho_1, \hat{\rho}_1) \leq \E\big[\big|X_{0, 1}(x_0) - \hat{X}_{0, 1}(x_0)\big|^2\big],
	\end{equation}
	because the left-hand side is the infimum over all couplings and the right-hand side corresponds to a specific coupling that pairs points from the same initial condition. This completes the proof.
\end{proof}

\eulerianerr*
\begin{proof}
	We first define the error metric
	\begin{equation}
		E_{s, t} = \E\left[\big|X_{s, t}(I_s) - \hat{X}_{s, t}(I_s)\big|^2\right].
	\end{equation}
	It then follows by direct differentiation that
	\begin{equation}
		\begin{aligned}
			\partial_s E_{s, t} & = \E\left[2\left(X_{s, t}(I_s) - \hat{X}_{s, t}(I_s)\right)\cdot\left(\partial_s X_{s, t}(I_s) + \dot{I}_s\cdot\nabla X_{s, t}(I_s) - \left(\partial_s \hat{X}_{s, t}(I_s) + \dot{I}_s \cdot \nabla \hat{X}_{s, t}(I_s)\right)\right)\right], \\
			                    & = \E\left[2\left(X_{s, t}(I_s) - \hat{X}_{s, t}(I_s)\right)\cdot\left(\partial_s X_{s, t}(I_s) + b_s(I_s)\cdot\nabla X_{s, t}(I_s) - \left(\partial_s \hat{X}_{s, t}(I_s) + b_s(I_s) \cdot \nabla \hat{X}_{s, t}(I_s)\right)\right)\right],   \\
			                    & \geq -E_{s, t} - \E\left[\left|\partial_s X_{s, t}(I_s) + b_s(I_s)\cdot\nabla X_{s, t}(I_s) - \left(\partial_s \hat{X}_{s, t}(I_s) + b_s(I_s) \cdot \nabla \hat{X}_{s, t}(I_s)\right)\right|^2\right],                                        \\
			                    & = -E_{s, t} - \delta_{s, t}^{\emd}.
			\nonumber
		\end{aligned}
	\end{equation}
	Above, we used the tower property of the conditional expectation, the Eulerian equation $\partial_s X_{s, t}(I_s) + b_s(I_s)\cdot\nabla X_{s, t}(I_s) = 0$, and defined the two-time Eulerian distillation error,
	\begin{equation}
		\delta_{s, t}^{\emd} = \E\left[\left|\partial_s \hat{X}_{s, t}(I_s) + b_s(I_s) \cdot \nabla \hat{X}_{s, t}(I_s)\right|^2\right].
	\end{equation}
	This implies that
	\begin{equation}
		\partial_s \left(-e^{s}E_{s, t}\right) \leq e^{s}\delta_{s t}^{\emd}.
	\end{equation}
	Using that $E_{t, t} = 0$ for any $t \in [0, T]$ and integrating with respect to $s$ from $s$ to $t$,
	\begin{equation}
		-e^{t}E_{t, t} + e^sE_{s, t} \leq \int_s^t e^{u}\delta_{u, t}^{\emd}du.
	\end{equation}
	It then follows that
	\begin{equation}
		E_{s, t} \leq \int_s^t e^{u-s}\delta_{u, t}^{\emd}du,
	\end{equation}
	and hence, using that $u-s \in [0, t]$ and that $\delta_{u, t}^{\emd} \geq 0$,
	\begin{equation}
		\E\big[\big|X_{0, 1}(x_0) - \hat{X}_{0, 1}(x_0)\big|^2\big] \le e^1\int_0^1 \E\left[\left|\partial_s \hat{X}_{s, 1}(I_s) + b_s(I_s) \cdot \nabla \hat{X}_{s, 1}(I_s)\right|^2\right]ds.
	\end{equation}
	The proof is completed upon noting that
	\begin{equation}
		W_2^2(\rho_1, \hat{\rho}_1) \leq \E\big[\big|X_{0, 1}(x_0) - \hat{X}_{0, 1}(x_0)\big|^2\big],
	\end{equation}
	because the left-hand side is the infimum over all couplings and the right-hand side corresponds to a particular coupling.
\end{proof}

\fmm*
\begin{proof} We start by noticing that
	\begin{equation}
		\label{eq:step1}
		\begin{aligned}
			 & \E \big [| \partial_t \hat X_{s,t} (\hat X_{t,s}(I_t)) - \dot{I}_t|^2,                                                                           \\
			 & = \E \big [| \partial_t \hat X_{s,t} (\hat X_{t,s}(I_t))|^2 - 2 \dot{I}_t \cdot \partial_t \hat X_{s,t} (\hat X_{t,s}(I_t)) + |\dot I_t|^2\big], \\
			 & = \E \big [| \partial_t \hat X_{s,t} (\hat X_{t,s}(I_t))|^2 - 2 \E[ \dot I_t|I_t]\cdot
			\partial_t \hat X_{s,t} (\hat X_{t,s}(I_t)) + |\dot I_t|^2\big],                                                                                    \\
			 & = \E \big [| \partial_t \hat X_{s,t} (\hat X_{t,s}(I_t))|^2 - 2 b_t(I_t)\cdot
				\partial_t \hat X_{s,t} (\hat X_{t,s}(I_t)) + |\dot I_t|^2\big],
		\end{aligned}
	\end{equation}
	where we used the tower property of the conditional expectation to get the third equality and the definition of $b_t(x)$ in~\eqref{eqn:b:def} to get the last. This means that the loss~\eqref{eq:obj:match} can be written as
	\begin{equation}
		\label{eq:obj:match:2}
		\begin{aligned}
			 & \calL_{\fmmlab}(\hat X)                                                                                                                                               \\
			 & = \int_{[0,1]^2} \int_{\R^d} w_{s,t} \big [| \partial_t \hat X_{s,t} (\hat X_{t,s}(x)) - b_t(x)|^2 + | \hat X_{s,t} (\hat X_{t,s}(x)) - x|^2 \big] \rho_t(x) dx ds dt \\
			 & \qquad + \int_{[0,1]^2} w_{s,t} \E\big[ |\dot I_t|^2 - |b_t(I_t)|^2\big] ds dt,
		\end{aligned}
	\end{equation}
	where $\rho_t = \Law(I_t)$. The second integral does not depend on $\hat X$, so it does not affect the minimization of $\calL_{\fmmlab}(\hat X)$. Assuming that $w_{s,t}>0$, the first integral is minimized if and only if we have
	\begin{equation}
		\label{eq:obj:match:3}
		\begin{aligned}
			\forall \:\: (s,t,x) \in [0,1]^2\times \R^d \ : \quad  \partial_t \hat X_{s,t} (\hat X_{t,s}(x)) = b_t(x) \quad \text{and} \quad \hat X_{s,t} (\hat X_{t,s}(x)) = x.
		\end{aligned}
	\end{equation}
	From the second of these equations it follows that: (i)  $\hat X_{s,s}(x) =x$, and (ii) if we evaluate the first equation at $y= \hat X_{t,s}(x)$, this equation reduces to
	\begin{equation}
		\label{eq:obj:mmm:sol}
		\forall \:\: (s,t,y) \in [0,1]^2\times \R^d \ : \quad\partial_t \hat X_{s,t}(y) = b_t(\hat X_{s,t}(y))
	\end{equation}
	which recovers~\eqref{eqn:flow_map_dyn:L}.
\end{proof}

\emm*

\begin{proof}
	The functional gradient of~\eqref{eq:obj:eul:match} with respect to $\hat{X}_{s, t}$ is given by
	\begin{equation}
		\begin{aligned}
			 & -\partial_s\big( w_{s,t} \big[ \partial_s \hat{X}_{s, t}(x) + \E\big[\dot{I}_s\cdot\nabla \hat{X}_{s, t}(I_s) \mid I_s = x\big] \big]\rho_s(x)\big) \\
			 & =  -\partial_s\big( w_{s,t} \big[\partial_s \hat{X}_{s, t}(x) + b_s(x) \cdot\nabla \hat{X}_{s, t}(x) \big]\rho_s(x)\big).
		\end{aligned}
	\end{equation}
	The flow map zeroes this quantity since it solves the Eulerian equation~\eqref{eqn:backward} that appears in it, and hence is a critical point of the objective.
\end{proof}

\fmmd*

\begin{proof}
	Equation~\eqref{eq:obj:match:D} is a PINN loss whose unique minimizer satisfies
	\begin{align}
		\label{eq:obj:match:D:s}
		\forall \:\: (s,t,x) \in [0,1]^2\times \R^d  \ :  \quad  \check X_{s,t} (x) = \big(\hat{X}_{t_{K-1}, t_K}\circ \cdots \circ \hat{X}_{t_1, t_2}\big)(x),
	\end{align}
	which establishes the claim.
\end{proof}


\section{Relation to existing consistency and distillation techniques}
\label{sec:app:relation}
In this section, we recast consistency models and several distillation techniques in the language of our two-time flow map $X_{s, t}$ to clarify their relation with our work.

\subsection{Flow maps and denoisers}
\label{se:fm:denoi}

Since $\Law(X_{t,s}(I_t)) = \Law(I_s)$, it is tempting to replace $X_{t,s}(I_t)$ by $I_s$ in the loss~\eqref{eq:obj:match} and use instead
\begin{equation}
	\label{eq:obj:match:denoise}
	\calL_{\text{denoise}}[\hat X] = \int_{[0,1]^2} w_{s,t} \E\big [| \partial_t \hat X_{s,t} (I_s) - \dot I_t|^2 \Big]  ds dt,
\end{equation}
minimized over all $\hat X$ such that $\hat X_{s,s}(x)=x$.
However, the minimizer of this objective is \textit{not} the flow map $X_{s,t}$, but rather the denoiser
\begin{equation}
	\label{eq:denoise}
	X^{\text{denoise}}_{s,t}(x) = \E [ I_t| I_s = x].
\end{equation}
This can be seen by noticing  that  the minimizer of~\eqref{eq:obj:match:denoise} is the same as the minimizer of
\begin{equation}
	\label{eq:obj:match:denoise:2}
	\begin{aligned}
		\calL'_{\text{denoise}}[\hat X] & = \int_{[0,1]^2} w_{s,t} \E\big [\big| \partial_t \hat X_{s,t} (I_s) - \E[\dot I_t|I_s]\big|^2 \Big]  ds dt,                      \\
		                                & = \int_{[0,1]^2} \int_{\R^d}w_{s,t} \Big [\big| \partial_t \hat X_{s,t} (x) - \E[\dot I_t|I_s=x]\big|^2 \Big] \rho_s(x) dx ds dt,
	\end{aligned}
\end{equation}
which follows from an argument similar to the one used in the proof of \Cref{prop:fmm}. The minimizer of~\eqref{eq:obj:match:denoise:2} satifies
\begin{equation}
	\label{eq:denoise:diff}
	\partial_t \hat X_{s,t}(x) = \E [ \dot I_t| I_s = x] = \partial_t \E [ I_t| I_s = x],
\end{equation}
which implies~\eqref{eq:denoise} by the boundary condition $\hat X_{s,s}(x) = x$.
The denoiser~\eqref{eq:denoise} may be useful, but it is not a consistent generative model.
For instance, if $x_0\sim\rho_0$ and $x_1\sim \rho_1$ are independent in the definition of $I_t$, since $I_0=x_0$ and $I_1=x_1$ by construction, for $s=0$ and $t=1$ we have
\begin{equation}
	\label{eq:denoise:onestep}
	X^{\text{denoise}}_{0,1}(x) = \E [x_1]
\end{equation}
i.e. the one-step denoiser only recovers the mean of the target density~$\rho_1$.

\subsection{Relation to consistency models}
\paragraph{Noising process.}
Following the recommendations in~\citet{karras_elucidating_2022} (which are followed by both~\citet{song_consistency_2023} and~\citet{song_improved_2023}), we consider the variance-exploding process\footnote{Oftentimes $t=0$ is set to $t=t_{\min}>0$ for numerical stability, choosing e.g. $t_{\min} = 2 \times 10^{-3}$. }
\begin{equation}
	\label{eqn:ve}
	\tilde x_t = a + t z, \:\:\: t \in [0, t_{\max}],
\end{equation}
where $a \sim \rho_1$ (data from the target density) and $z \sim \mathsf N(0, I)$.
In practice, practitioners often set $t_{\max} = 80$.
In this section, because we follow the score-based diffusion convention, we set time so that $t = 0$ recovers $\rho_1$ and so that a Gaussian is recovered as $t \to\infty$.
The corresponding probability flow ODE is given by
\begin{equation}
	\label{eqn:ve_pflow}
	\dot{\tilde x}_t = -t\nabla\log\rho_t(\tilde x_t), \qquad \tilde x_{t=0} = a \sim \rho_1
\end{equation}
where $\rho_t(x) = \Law(\tilde x_t)$.
In practice,~\eqref{eqn:ve_pflow} is solved backwards in time from some terminal condition $\tilde x_{t_{\max}}$.
To make contact with our formulation where time goes forward, we define ${x}_t = \tilde x_{\tmax-t}$, leading to
\begin{equation}
	\label{eqn:ve_fwd}
	\dot{x}_t = (\tmax-t)\nabla\log\rho_{\tmax-t}(x_t), \qquad x_{t=0} \sim \mathsf N(x_0,t_{\max}^2 I).
\end{equation}
To make touch with our flow map notation, we then define
\begin{equation}
	\label{eqn:ve_fwd_flow}
	\partial_t{X}_{s, t}(x) = (\tmax -t)\nabla\log\rho_{\tmax - t}(X_{s, t}(x)), \qquad X_{s,s}(x) = x.
\end{equation}

\paragraph{Consistency function.}
By definition~\citep{song_consistency_2023}, the consistency function $f_t: \R^d \to \R^d $ is such that
\begin{equation}
	\label{eqn:consistency_def}
	f_t(\tilde x_t) = a,
\end{equation}
where $\tilde x_t$ denotes the solution of~\eqref{eqn:ve_pflow} and $a\sim\rho_1$.
To make a connection with our flow map formulation, let us consider~\eqref{eqn:consistency_def} from the perspective of $x_t$,
\begin{equation}
	\label{eqn:consistency_def_tilde}
	f_t(x_{\tmax-t})= x_{\tmax},
\end{equation}
which is to say that
\begin{equation}
	\label{eqn:consistency_flow_map}
	f_t(x) = X_{\tmax-t, \tmax}(x).
\end{equation}
Note that only one time is varied here, i.e. $f_t(x)$, cannot be iterated upon: by its definition~\eqref{eqn:consistency_def}, it always maps the observation~$\tilde x_t$ onto a sample $a\sim\rho_1$.

\paragraph{Discrete-time loss function for distillation.}
In practice, consistency models are typically trained in discrete-time, by discretizing $[\tmin, \tmax]$ into a set of $N$ points $\tmin = t_1 < t_2 < \hdots < t_N = \tmax$.
According to~\citet{karras_elucidating_2022}, these points are chosen as
\begin{equation}
	\label{eqn:t_disc}
	t_i = \left(\tmin^{1/\eta} + \frac{i-1}{N-1}\left(\tmax^{1/\eta} - \tmin^{1/\eta}\right)\right)^\eta,
\end{equation}
with $\eta = 7$. Assuming that we have at our disposal a pre-trained estimate $s_t(x)$ of the score $\nabla \log \rho_t(x)$, the \textit{distillation loss} for the consistency function $f_t(x)$ is then given by
\begin{equation}
	\label{eqn:loss_consistency}
	\begin{aligned}
		\calL_{\mathsf{CD}}^N(\hat f) & = \sum_{i=1}^{N-1} \E\Big[\big|\hat f_{t_{i+1}}(\tilde x_{t_{i+1}}) - \hat f_{t_i}(\hat{x}_{t_i})\big|^2\Big], \\
		\tilde x_{t_{i+1}}            & = a + t_{i+1}z                                                                                                 \\
		\hat{x}_{t_i}                 & = \tilde x_{t_{i+1}} - \left(t_{i} - t_{i+1}\right)t_{i+1}s_{t_{i+1}} (x_{t_{i+1}}),
	\end{aligned}
\end{equation}
where  $\E$ is taken over the data $a \sim \rho_1$ and $z \sim \mathsf N(0, I)$.
The term $\hat{x}_{t_i}$ is an approximation of $\tilde x_{t_i}$ computed by taking a single step of~\eqref{eqn:ve_pflow} with the approximate score model~$s_t(x)$. In practice, the square loss in~\eqref{eqn:loss_consistency} can be replaced by an arbitrary metric $d:\R^d\rightarrow\R^d\rightarrow\R_{\geq 0}$, such as a learned metric like LPIPS or the Huber loss.

\paragraph{Continuous-time limit.}
In continuous-time, it is easy to see via Taylor expansion that the consistency loss reduces to
\begin{equation}
	\label{eqn:CD_continuous}
	\calL^\infty_{\mathsf{CD}}(\hat f) = \lim_{N\to\infty} N \calL_{\mathsf{CD}}^N(\hat f) =\int_{\tmin}^{\tmax}\int_{\R^d} w^2_t \big|\partial_t f_t(x) - t s_t(x)\cdot \nabla f_t(x)\big|^2\rho_t(x)dxdt,
\end{equation}
where $w_t = \eta (\tmax^{1/\eta}-\tmin^{1/\eta}) t^{1-1/\eta}$ is a weight factor arising from the nonuniform time-grid.
This is a particular case of our Eulerian distillation loss~\eqref{eqn:L_backward} applied to the variance-exploding setting~\eqref{eqn:ve} with the identification~\eqref{eqn:consistency_flow_map}.

\paragraph{Estimation vs distillation of the consistency model.}
If we approximate the exact
\begin{equation}
	\label{eqn:score_exact_ve}
	\nabla\log\rho_t(x) = -\E\left[\frac{\tilde x_t - a}{t^2} \Big| \tilde x_{t} = x\right],
\end{equation}
by a single-point estimator
\begin{equation}
	\label{eqn:score_approx_ve}
	\nabla\log\rho_t(x) \approx \frac{a - \tilde x_t}{t^2},
\end{equation}
we may use the expression
\begin{equation}
	\label{eqn:approx_xhat}
	\hat{x}_{t_i} \approx a + t_{i}z,
\end{equation}
in~\eqref{eqn:loss_consistency} to obtain the \textit{estimation} loss,
\begin{equation}
	\label{eqn:loss_consistency_training}
	\begin{aligned}
		\calL_{\mathsf{CT}}^N(\hat f) & = \sum_{i=1}^{N-1}\E\Big[\big|\hat f_{t_{i+1}}(\tilde x_{t_{i+1}}) - \hat f^{-}_{t_i}(\tilde x_{t_i})\big|^2\Big], \\
		\tilde x_{t_i}                & = a + t_i z.
	\end{aligned}
\end{equation}
This expression does not require a previously-trained score model.
Notice, however, that \eqref{eqn:loss_consistency_training} must be used with a $\mathsf{stopgrad}$ on $f^-_{t_i}(\tilde x_{t_i})$ so that the gradient is taken over only the first $\hat f_{t_{i+1}}(\tilde x_{t_{i+1}})$.
This is because~\eqref{eqn:loss_consistency} and \eqref{eqn:loss_consistency_training} are different objectives with different minimizers, even at leading order after expansion in $1/N$, for the same reason that~\eqref{eqn:L_backward} differs from \eqref{eqn:L_backward:2}.
To see this, observe that to leading order,
\begin{equation}
	\label{eq:expand:fstop}
	\hat f^{-}_{t_i}(\tilde x_{t_i}) = \hat f^{-}_{t_{i+1}}(\tilde x_{t_{i+1}}) + \big( \partial_t \hat f^{-}_{t_{i+1}}(\tilde x_{t_{i+1}}) + z \cdot \nabla f^{-}_{t_{i+1}}(\tilde x_{t_{i+1}})\big) (t_{i}-t_{i+1}) + O\big((t_{i}-t_{i+1})^2\big),
\end{equation}
which gives the continuous-time limit
\begin{equation}
	\label{eqn:CT_continuous}
	\calL^\infty_{\mathsf{CT}}(\hat f) = \lim_{N\to\infty} \calL_{\mathsf{CD}}^N(\hat f) =\int_{\tmin}^{\tmax}\int_{\R^d} w_t \big|\partial_t f_t(x) + z \cdot \nabla f^-_t(x)\big|^2\rho_t(x)dxdt.
\end{equation}
Observing that $z = \partial_t \tilde{x}_t$ shows that~\eqref{eqn:CT_continuous} recovers the Eulerian estimator described in~\cref{sec:eul:est}, which does not lead to a gradient descent iteration.

\subsection{Relation to neural operators}
In our notation, neural operator approaches for fast sampling of diffusion models~\citep{zheng_fast_2023} also estimate the flow map $X_{0,t}$ via the loss
\begin{equation}
	\label{eqn:fno_loss}
	\calL_{\mathsf{FNO}}(\hat X) = \int_0^1\int_{\R^d} \big|\hat X_{0, t}(x) - X_{0, t}(x)\big|^2 \rho_0(x) dx dt,
\end{equation}
where $\hat X_{0, t}$ is parameterized by a Fourier Neural Operator and where $X_{0, t}$ is the flow map \textit{obtained by simulating the probability flow ODE associated with a pre-trained (or given)} $b_t(x)$.
To avoid simulation at learning time, they pre-generate a dataset of trajectories, giving access to $X_{0, t}(x)$ for many initial conditions $x\sim \rho_0$.
Much of the work focuses on the architecture of the FNO itself, which is combined with a U-Net.

\subsection{Relation to progressive distillation}
Progressive distillation~\citep{salimans_progressive_2022} takes a DDIM sampler~\citep{song_denoising_2022} and trains a new model to approximate two steps of the old sampler with one step of the new model.
This process is iterated repeatedly to successively halve the number of steps required.
In our notation, this corresponds to minimizing
\begin{equation}
	\label{eqn:prog_dist}
	\calL_{\mathsf{PD}}^{\Delta t}(\hat X) = \int_0^{1-2\Delta t}\int_{\R^d}\big|\hat X_{t, t+2\Delta t}(x) - \big(X_{t + \Delta t, t+2\Delta t}\circ X_{t, t+\Delta t}\big)(x)\big|^2\rho_t(x) dx dt
\end{equation}
where $X$ is a pre-trained map from the previous iteration.
This is then iterated upon, and $\Delta t$ is increased, until what is left is a few-step model.

\begin{table}[t]
	\centering
	\begin{tabular}{lcccc}
		\toprule
		            & $\mathsf{KL}(\rho_1 || \hat \rho_1)$ & $W_2^2(\rho_1, \hat \rho_1)$ & $W_2^2(\hat{\rho}_1^b, \hat \rho_1)$ & $L_2$ error \\
		\midrule
		SI          & 0.020                                & 0.026                        & 0.0                                  & 0.000       \\
		LMD         & 0.043                                & 0.059                        & 0.032                                & 0.085       \\
		EMD         & 0.079                                & 0.029                        & 0.010                                & 0.011       \\
		FMM,  $N=1$ & 0.104                                & 0.021                        & --                                   & 0.026       \\
		FMM,  $N=4$ & 0.045                                & 0.014                        & --                                   & 0.024       \\
		PFMM, $N=1$ & 0.043                                & 0.014                        & --                                   & 0.023       \\
		\bottomrule
	\end{tabular}
	\captionof{table}{Comparison of $\mathsf{KL}(\rho || \hat \rho)$ and $W_2^2(\rho, \hat \rho)$, where  $\hat\rho$ is the pushforward density from the maps $\hat X_{0,1}(x_0)$ for the methods listed above. Additionally included is a comparison of  $L_2$ expected error of the distillation methods against their teacher $\hat X_{0,1}^{\si}$ given as $\mathbb E[ | \hat X_{0,1}^{\si}(x) - \hat X_{0,1}(x)|^2] $. Intriguingly, LMD performs better in being distributionally correct, as measured by the KL-divergence, but worse in preserving the coupling of the teacher model. The roles are flipped for EMD. This may highlight KL as a more informative metric in our case, as our aims are to sample correctly in distribution. See Figure \ref{fig:mislabel} for a visualization.}
	\label{tab:comparison:checker}
\end{table}%

\section{Additional Experimental Details}
\label{app:exp:checker}

\subsection{2D checkerboard}
Here, we provide further discussion and analysis of our results for generative modeling on the 2D checkerboard distribution (\cref{fig:checker}).
Our KL-divergence estimates clearly highlight that there is a hierarchy of performance.
Of particular interest is the large discrepancy in performance between the Eulerian and Lagrangian distillation techniques.

As noted in Figure \ref{fig:checker} and Table \ref{tab:comparison:checker}, LMD substantially outperforms its Eulerian counterpart in terms of minimizing the KL-divergence between the target checkerboard density $\rho_1$ and model density $\hat \rho_1 = \hat{X}_{0, 1}\sharp\rho_0$.
Interestingly, while LMD is more correct in distribution, EMD better preserves the original coupling $(x_0, \hat X_{0,1}^{\si}(x_0) )$ of the teacher model $\hat{X}_{0,1}^{\si}$, as measured by the $W_2^2$ distance and the expected $L_2$ reconstruction error, defined as $\mathbb E[ | \hat X_{0,1}^{\si}(x) - \hat X_{0,1}(x)|^2] $. Where this coupling is significantly \textit{not} preserved is visualized in Figure \ref{fig:mislabel}. For each model, we color code points for which $| \hat X_{0,1}^{\si}(x) - \hat X_{0,1}(x)|^2 > 1.0$, highlighting where the student map differed from the teacher. We notice that the LMD map pushes initial conditions to an opposing checker edge (purple) than where those initial conditions are pushed by the interpolant (blue). This is much less common for the EMD map, but its performance is overall worse in matching the target distribution.

\begin{figure}[h]
	\centering
	\includegraphics[width=\textwidth]{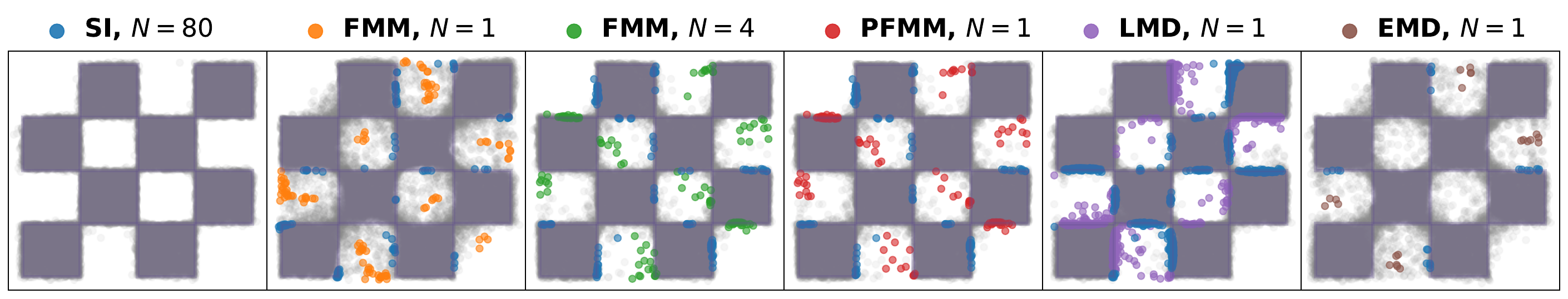}
	\caption{Visualization of the difference in assignment of the maps $\hat X_{0,1}(x_0)$ for the various models as compared to the teacher/ground truth model $\hat X^{\si}_{0,1}(x_0)$ for the same initial conditions from the base distribution $x_0$.  Points that lie in the region $|\hat X^{\si}_{0,1}(x_0) - \hat X_{0,1}(x_0) |^2 > 1.0$ are colored as compared to the blue points, which represent where the stochastic interpolant teacher mapped the same red initial conditions. This gives us an intuition for how well each method precisely maintains the coupling $(x_0, \hat X_{0,1}^{\si}(x_0) ) $ from the teacher. Note that we are treating $X_{0,1}^{\si}$ as the ground truth map here, as it is close to the exact map. The models based on FMM either don't have a teacher or have $\fmmlab, N=4$ as a teacher, but all should have the same coupling at the minimizer. }
	\label{fig:mislabel}
\end{figure}

\subsection{Image experiments}

Here we include more experimental details for reproducing the results provided in Section \ref{sec:exp}. We use the U-Net from the diffusion OpenAI paper \citep{dhariwal_diffusion_2021} with code given at \href{https://github.com/openai/guided-diffusion}{https://github.com/openai/guided-diffusion}. We use the same architecture for both CIFAR10 and ImageNet-$32\times32$ experiments. The architecture is also the same for training a velocity field and for training a flow map, barring the augmentation of the time-step embedding in the U-Net to handle two times ($s,t$) instead of one. Details of the training conditions are presented in Table \ref{tab:archs:img}.

\begin{table}[ht]
	\centering
	\begin{tabular}{lcc}
		\toprule                                       & CIFAR-10          & ImageNet 32$\times$32 \\
		\midrule Dimension                             & 32$\times$32      & 32$\times$32          \\
		\# Training point                              & $5\times 10^4$    & 1,281,167             \\
		\midrule Batch Size                            & 256               & 256                   \\
		Training Steps (Lagrangian distillation)       & 1.5$\times$10$^5$ & 2.5$\times$10$^5$     \\
		Training Steps (Eulerian distillation)         & 1.2$\times$10$^5$ & 2.5$\times$10$^5$     \\
		Training Steps (Flow map matching)             & N/A               & 1$\times$10$^5$       \\
		Training Steps (Progressive flow map matching) & 1.3$\times 10^5$  & N/A                   \\
		U-Net channel dims                             & 256               & 256                   \\
		Learning Rate (LR)                             & $0.0001$          & $0.0001$              \\
		LR decay (every 1k epochs)                     & 0.992             & 0.992                 \\
		U-Net dim mult                                 & [1,2,2,2]         & [1,2,2,2]             \\
		Learned time  embedding                        & Yes               & Yes                   \\
		$\#$ GPUs                                      & 4                 & 4                     \\
		\bottomrule
	\end{tabular}
	\vspace{2.5mm}
	\caption{Hyperparameters and architecture for image datasets.}
	\label{tab:archs:img}
\end{table}

\end{document}